\documentclass[journal]{IEEEtai}

\usepackage[colorlinks,urlcolor=blue,linkcolor=blue,citecolor=blue]{hyperref}

\usepackage{color,array}

\usepackage{graphicx}
\usepackage{booktabs}
\usepackage{url}
\usepackage{graphicx}
\usepackage{subfigure}
\usepackage{amsmath}
\usepackage{amssymb}
\usepackage{mathtools}
\usepackage{amsthm}
\usepackage{multirow}
\usepackage{cite}
\usepackage[ruled]{algorithm2e}
\usepackage{orcidlink}
\usepackage{etoolbox}
\usepackage{balance}

\providetoggle{beginFlag}
\settoggle{beginFlag}{true}
  
\SetAlgoSkip{SkipBeforeAndAfter}

\newtheorem{theorem}{Theorem}[section]

\newtheorem{lemma}[theorem]{Lemma}

\theoremstyle{definition}

\begin{document}

\title{Critical Sampling for Robust Evolution Operator Learning of Unknown Dynamical Systems}

\author{Ce Zhang$^{\orcidlink{0000-0001-6789-0130}}$, \IEEEmembership{Student Member, IEEE}, Kailiang Wu$^{\orcidlink{0000-0002-9042-3909}}$, and Zhihai He$^{\orcidlink{0000-0002-2647-8286}}$, \IEEEmembership{Fellow, IEEE}
% \thanks{This paragraph of the first footnote will contain the date on which you submitted your paper for review. It will also contain support information, including sponsor and financial support acknowledgment. For example, ``This work was supported in part by the U.S. Department of Commerce under Grant BS123456.'' }
\thanks{Corresponding author: Zhihai He.}
\thanks{C. Zhang is with the Machine Learning Department, Carnegie Mellon University, Pittsburgh, United States. Most of this work was done during his undergraduate studies at Department of Electronic and Electrical Engineering, Southern University of Science and Technology, Shenzhen, P. R. China (email: cezhang@cs.cmu.edu).}
\thanks{K. Wu is with the Department of Mathematics, Southern University of Science and Technology, Shenzhen, P. R. China (email: wukl@sustech.edu.cn).}
\thanks{Z. He is with the Department of Electronic and Electrical Engineering, Southern University of Science and Technology, Shenzhen, P. R. China, and also with Pengcheng Laboratory, Shenzhen, P. R. China (email: hezh@sustech.edu.cn).}
% \thanks{This paragraph will include the Associate Editor who handled your paper.}
}

\markboth{IEEE Transactions on Artificial Intelligence, October 2023}
{Zhang \MakeLowercase{\textit{et al.}}: Critical Sampling for Robust Evolution Operator Learning of Unknown Dynamical Systems}

\maketitle
\begin{abstract}
Given an unknown dynamical system, what is the minimum number of samples needed for effective learning of its governing laws and accurate prediction of its future evolution behavior, and how to select these critical samples? In this work, we propose to explore this problem based on a design approach. Starting from a small initial set of samples, we adaptively discover critical samples to achieve increasingly accurate learning of the system evolution. One central challenge here is that we do not know the network modeling error since the ground-truth system state is unknown, which is however needed for critical sampling. To address this challenge, we introduce a multi-step reciprocal prediction network where forward and backward evolution networks are designed to learn the temporal evolution behavior in the forward and backward time directions, respectively. Very interestingly, we find that the desired network modeling error is highly correlated with the multi-step reciprocal prediction error, which can be directly computed from the current system state. This allows us to perform a dynamic selection of critical samples from regions with high network modeling errors for dynamical  systems. Additionally, a joint spatial-temporal evolution network is introduced which incorporates spatial dynamics modeling into the temporal evolution prediction for robust learning of the system evolution operator with few samples. Our extensive experimental results demonstrate that our proposed method is able to dramatically reduce the number of samples needed for effective learning and accurate prediction of evolution behaviors of unknown dynamical systems by up to hundreds of times.
\end{abstract}

\begin{IEEEImpStatement}
\looseness=-1
This work investigates an important problem in learning and prediction with deep neural networks: \textit{how to characterize and estimate the prediction errors during the inference stage?} This is a very challenging problem since the ground-truth values for unknown dynamical systems are not available. In recent years, learning-based methods for complex and dynamic system modeling have become an important area of research in artificial intelligence. We recognize that many existing data-driven approaches for learning the evolution operator typically assume the availability of sufficient data. In practice, this is not the case. In practical dynamical systems, such as ocean, cardiovascular and climate systems, sample collection is often costly or very limited due to resource constraints or experimental accessibility. To our best knowledge, this work is one of the first efforts to address this challenge. The success of our approach also contributes significantly to deep learning and signal estimation research. 
\end{IEEEImpStatement}


\section{Introduction}
\label{sec-intro}
\IEEEPARstart{R}{ecently}, learning-based methods for complex and dynamic system modeling have become an important area of research in machine learning \cite{raissi2018deep,hernandez2022thermodynamics,xing2022neural}.
The behaviors of dynamical systems in the physical world are governed by their underlying  physical laws \cite{bongard2007automated,schmidt2009distilling}. In many areas of science and engineering, ordinary differential equations (ODEs) and partial differential equations (PDEs) play important roles in describing and modeling these physical laws \cite{brunton2016discovering,raissi2018deep,long2018pde,chen2018neural,raissi2019physics,qin2019data}. 
In recent years, data-driven modeling of unknown physical systems from measurement data has emerged as an important area of research.
There are two major approaches that have been explored.
The first approach typically tries to identify all the potential terms in the unknown governing equations from a priori dictionary, which includes all possible terms that may appear in the equations \cite{brunton2016discovering,schaeffer2017sparse,rudy2017data,raissi2018deep,long2018pde,Wu_2019_JCP,Wu_2020_SISC,xu2021robust,nazari2022physics}. The second approach for data-driven learning of unknown dynamical systems is to approximate the evolution operator of the underlying equations, instead of identifying the terms in the equations \cite{qin2019data,Wu_2020_JCP,qin2021data,li2021physics}.

Many existing data-driven approaches for learning the evolution operator typically assume the availability of sufficient data, and often 
require a large set of measurement samples to train the neural network, especially for high-dimensional systems. For example, to effectively learn a neural network model for the 2D Damped Pendulum ODE system, existing methods typically need more than 10,000 samples to achieve sufficient accuracy \cite{qin2019data, Wu_2020_JCP}. This number increases dramatically with the dimensions of the system. For example, for the 3D Lorenz system, the number of needed samples used in the literature is often increased to one million.
We recognize that, in practical dynamical systems, such as ocean, cardiovascular and climate systems, it is very costly to collect observation samples. 
This leads to a new and important research question:  \textit{what is the minimum number of samples needed for robust learning of the governing laws of an unknown system and accurate prediction of its future evolution behavior?} 

\begin{figure*}[!t]
\centering
\centering \includegraphics[width=0.85\linewidth]{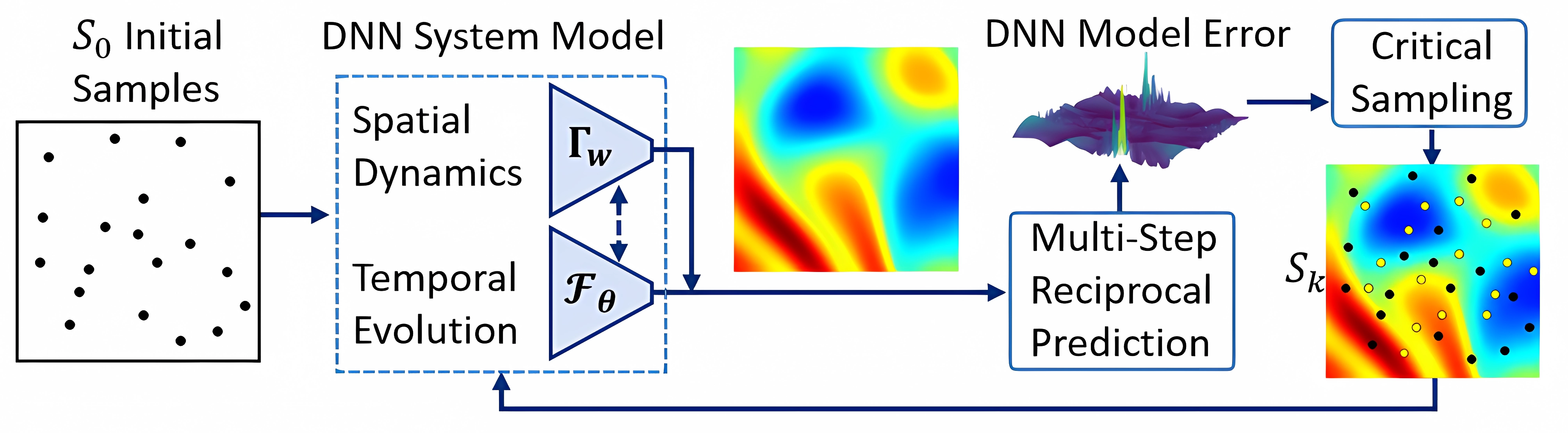}
\caption{Illustration of the proposed method of critical sampling for accurately learning the evolution behaviors of unknown dynamical systems.}
\label{fig:framework}
\vspace{-10pt}
\end{figure*}

Adaptive sample selection for network learning, system modeling and identification  has been studied in the areas of active learning and optimal experimental design \cite{gao2022failure,mao2023physics,wu2023comprehensive}.
Methods have been developed for global optimization of experimental sequences \cite{llamosi2014experimental}, active data sample generation for time-series learning and modeling \cite{zimmer2018safe}, 
Kriging-based sampling method for learning spatio-temporal dynamics of systems \cite{huang2022physics}, adaptive training of physics-informed deep neural networks \cite{zhang2022simulation}, and data-collection scheme for system identification \cite{mania2022active}. 
However, within the context of deep neural network modeling of unknown dynamical systems, the following key challenging issues have not been adequately addressed: (1) How to characterize and estimate the prediction error of the deep neural networks? (2) Based on this error modeling, how to adaptively select the critical samples and successfully train the deep neural networks from these few samples?

Figure \ref{fig:framework} illustrates the proposed method of critical sampling for accurately learning the evolution behaviors of unknown dynamical systems. 
We start with a small set of initial samples, then iteratively discover and collect critical samples to obtain more accurate network modeling of the system. During critical sampling, the basic rule is to select the samples from regions with high network modeling errors so that these selected critical samples can maximally reduce the overall modeling error. 
However, the major challenge here is that we do not know network modeling error,  i.e., the difference between the system state predicted by the network and the ground-truth which is not available for unknown systems.
To address this challenge, we establish  a multi-step reciprocal prediction framework where a forward evolution network and a backward evolution network are designed to learn and predict the temporal evolution behavior in the forward  and backward time directions, respectively. 
Our hypothesis is that, if the forward and backward prediction models are both accurate, starting from an original state $A$, if we perform the forward prediction for $K$ times and then perform the backward prediction for another $K$ times, the final prediction result $\bar{A}$ should match the original state $A$. The error between $\bar{A}$ and $A$ is referred to as the \textit{multi-step reciprocal prediction error}.

Very interestingly, we find that the network modeling error is highly correlated with the multi-step reciprocal prediction error. Note that multi-step reciprocal prediction error  can be directly computed from the current system state, without the need to know the ground-truth system state\footnote{"The ground-truth system state" refers to the true or actual state of the system at a specific time. It represents the accurate and precise system state that is obtained by highly accurate ODE/PDE solvers.}. 
This allows us to perform a dynamic selection of critical samples from regions with high network modeling errors and develop an adaptive learning method for dynamical systems.
To effectively learn the system evolution from this small set of critical samples, we introduce a joint spatial-temporal evolution network structure which couples spatial dynamics learning with temporal evolution learning. 
Our extensive experimental results demonstrate that our proposed method is able to dramatically reduce the number of samples needed for effective learning and accurate prediction of evolution behaviors of unknown dynamical systems. 
This paper has significant impacts in practice since collecting samples from real-world dynamical systems can be very costly or limited due to resource/labor constraints or experimental accessibility.

The \textbf{major contributions} of this work can be summarized as follows. 
(1) We have successfully developed a multi-step reciprocal prediction approach to characterize the prediction errors in deep neural network modeling of unknown dynamical systems. We have made an interesting finding that the network modeling error is highly correlated with the multi-step reciprocal prediction error, which enables us to develop the critical sampling method. 
(2) We have incorporated spatial dynamics modeling into the temporal evolution prediction for sample augmentation, which enables us to predict or interpolate more samples at unknown locations for stable evolution operator learning.
(3) Our proposed method is able to dramatically improve the long-term prediction accuracy, while reducing the number of needed samples and related sample collection costs for learning the system evolution, which is highly desirable in practical applications.

\vspace{-1pt}
\section{Related Work}
\vspace{-1pt}
\label{sec-related}
In this section, we review existing research closely related to our work. 

\textbf{(1) Data-driven modeling of unknown physical systems.} There are two major approaches that have been explored in the literature.
The first approach aims to learn the mathematical formulas or expressions of the underlying governing equations. 
In a series developments of this direction, the seminal work was made by 
\cite{bongard2007automated,schmidt2009distilling}, where 
 symbolic regression was proposed to learn nonlinear dynamic systems from data. 
Later, more approaches have been proposed in this direction, including but not limited to sparse regression  \cite{brunton2016discovering,schaeffer2017sparse,rudy2017data}, neural networks \cite{long2018pde,raissi2018deep,long2019pde,nazari2022physics}, 
 Koopman theory \cite{brunton2017chaos}, Gaussian process regression \cite{raissi2017machine}, model selection approach \cite{mangan2017model}, classical polynomial approximations \cite{Wu_2019_JCP,Wu_2020_SISC},  genetic algorithms \cite{xu2020dlga,xu2021robust,xu2021deep}, and linear multi-step methods \cite{keller2021discovery}, etc.

The second approach aims to approximate the evolution operator of the underlying dynamical system typically via a deep neural network, which predicts the system state for the next time instance from the current state \cite{Wu_2020_JCP,qin2021data,li2021fourier,li2021physics,li2022learning}. In fact, the idea of such an approach is essentially equivalent to learn the integral form of the underlying unknown differential equations \cite{qin2019data}.
The performance of this approach has been demonstrated for learning ODEs \cite{qin2019data} and modeling PDEs in 
generalized Fourier spaces \cite{Wu_2020_JCP} and physical space \cite{chen2022deep}. 
Recently, this approach has also been extended to data-driven modeling of parametric differential equations \cite{qin2021deep}, non-autonomous systems \cite{qin2021data},   
partially observed systems \cite{fu2020learning}, biological models \cite{su2021deep}, and model correction \cite{chen2021generalized}. 
For an autonomous dynamical system, its evolution operator completely characterizes the system evolution behavior. Researchers have demonstrated that the evolution operator, once successfully learned, can be called repeatedly to predict the  evolution behaviors of the unknown dynamical systems \cite{qin2019data,Wu_2020_JCP,chen2022deep}.

In this work, we choose the second approach of learning the evolution operator.
Compared to the first approach  which tries to recover the mathematical expression of the unknown governing equations, the second approach of learning evolution operators often has the following distinctive features: 
(1) It does not require a prior knowledge about the form of the unknown governing equations.
(2) Unlike the first approach, the evolution operator approach, which is based on the integral form of the underlying dynamical system, does not require numerical approximation of the temporal derivatives and allows large time steps during the learning and prediction processes.
(3) Although the first approach may successfully recover the expressions of the governing equations, during the prediction stage, it still needs to  construct suitable numerical schemes to further solve the learned equations. On the contrary, the second approach learns the evolution operator which can be directly used to perform long-term prediction of the system  behavior in the future.

\textbf{(2) State-space models and adjoint state methods.}
State-space models have shown to be a powerful tool for modeling the behaviors of dynamical systems \cite{mcgoff2015statistical}. Methods have been developed for approximating dynamical systems with hidden Markov model (HMM) \cite{rabiner1989tutorial,fraser2008hidden}, recurrent neural network (RNN) \cite{han2004modeling}, long short-term memory network (LSTM) \cite{vlachas2018data}, reservoir computing (RC) \cite{inubushi2017reservoir}, structured variational autoencoder (SVAE) \cite{johnson2016composing}, linear dynamical system (LDS) and its variations \cite{fox2008nonparametric,linderman2017bayesian,gao2016linear}.

We recognize that the multi-step forward and backward processes and the usage of mismatch errors are related to those in the recent adjoint state methods for neural ODE learning (\textit{e.g.}~\cite{chen2018neural, zhuang2020mali, zhuang2020adaptive}). However, our method is uniquely different in the following aspects. (1) In the adjoint state method, the backpropagation is used to compute gradients based on a Lagrangian functional. However, in our method, the backward network is used to learn the inverse of the forward evolution operator (namely, backward evolution operator, see Lemma \ref{lem:backisforward}). (2) The adjoint state method aims to compute the gradients of the loss functions more efficiently and accurately. However, our method aims to discover the critical samples for network learning.

\vspace{-1pt}
\section{Method}
\vspace{-1pt}
\label{sec-method}

In this section, we present our method of critical sampling for accurate learning of the evolution behaviors for unknown dynamical systems.

\subsection{Problem Formulation and Method Overview}
\label{sec:overview}

In this work, we focus on learning the evolution operator $\mathbf{\Phi}_\Delta: \mathbb{R}^n\rightarrow \mathbb{R}^n$ for autonomous dynamical systems, which maps the system state from time $t$ to its next state at time $t+\Delta$: $    {\bf u}(t+\Delta) = \mathbf{\Phi}_\Delta({\bf u}(t)).$
It should be noted that, for autonomous systems, this evolution operator
$\mathbf{\Phi}_\Delta$ remains invariant over time. It only depends on the time difference $\Delta$.
For an autonomous system, its evolution operator completely characterizes the system evolution behavior \cite{qin2019data,Wu_2020_JCP,chen2022deep}.

Our  goal is to develop a deep neural network method to accurately learn the evolution operator and robustly predict the long-term evolution of the system using a minimum number of selected critical samples. 
Specifically, to learn the system evolution over time,  the measurement samples for training the evolution network are collected in the form of pairs. Each pair represents two solution states along the evolution trajectory at time instances $t$ and $t+\Delta$. For simplicity, we assume that the start time is $t=0$.  
Using a high-accuracy system solver, we generate $J$ system state vectors $\{{\bf u}^j(0)\}_{j=1}^J$ at time 0  and $\{{\bf u}^j(\Delta)\}_{j=1}^J$ at time $\Delta$
in the computational domain $D$.
Thus, the training samples are given by 
\begin{equation}
    \mathbf{\mathcal{S}}_F = \{[{\bf u}^j(0)\rightarrow {\bf u}^j(\Delta)] : {\bf u}^j(0), {\bf u}^j(\Delta)\in \mathbb{R}^n, 1\le j\le J\}.
\end{equation}
 It is used to train the forward evolution network $\mathcal{F}_\theta$ which  approximates the forward evolution operator $\mathbf{\Phi}_\Delta$.
 As discussed in Section \ref{sec-intro}, we introduce the idea of backward evolution operator  $\mathbf{\Psi}_\Delta: \mathbb{R}^n\rightarrow\mathbb{R}^n$, $\mathbf{u}(t) = \mathbf{\Psi}_\Delta(\mathbf{u}(t+\Delta))$.
 The original training samples in $\mathbf{\mathcal{S}}_F$ can be switched in time to create the following sample set 
 \begin{equation}
    \mathbf{\mathcal{S}}_G = \{[{\bf u}^j(\Delta)\rightarrow {\bf u}^j(0)]: 1\le j\le J\},
\end{equation}
 which is used to train the backward evolution network $\mathcal{G}_\vartheta $ to approximate the backward evolution operator $\mathbf{\Psi}_\Delta$.
 The forward and backward evolution networks, 
 $\mathcal{F}_\theta$ and  $\mathcal{G}_\vartheta$, allow us to iteratively predict the system's evolution in both forward and backward directions. 

\begin{figure*}[ht]
\vskip -0.05in
\centering
\centering \includegraphics[width=0.99\linewidth]{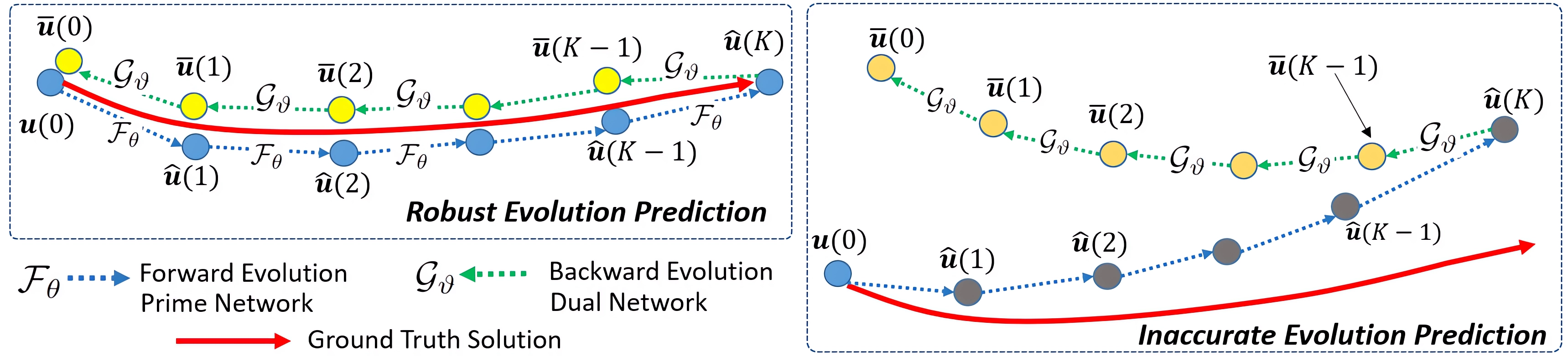}
\vskip -0.1in
\caption{Illustration of the proposed idea of multi-step reciprocal prediction error. The left figure shows the robust evolution prediction on locations with small multi-step reciprocal prediction errors, where the forward and backward paths align perfectly. The right figure shows the inaccurate evolution prediction on locations with large multi-step reciprocal prediction errors, where we need to select more critical samples.}
\label{fig:idea}
\vskip -0.15in
\end{figure*}

Based on the forward and backward evolution networks, we introduce the multi-step reciprocal prediction process. As illustrated in Figure \ref{fig:idea}, starting from the initial condition
$\mathbf{u}(0)\in \mathbb{R}^n$ at time $t=0$, we perform $K$-step prediction of the system state by repeatedly calling the forward evolution network $\mathbf{\cal F}_\theta$ with
$\hat{\mathbf{u}}((k+1)\Delta) = \mathbf{\cal F}_\theta[\hat{\mathbf{u}}(k\Delta)]$ and $\hat{\mathbf{u}}(0) = \mathbf{u}(0)$.
We then apply the backward evolution network $\mathcal{G}_\vartheta$ to perform $K$-step backward prediction: $\bar{\mathbf{u}}((k-1)\Delta) = \mathbf{\cal G}_\vartheta[\bar{\mathbf{u}}(k\Delta)]$
and get back to the initial condition { $\bar{\mathbf{u}}(0)$}. 
This  process of forward and backward evolution prediction is referred to as \textit{multi-step reciprocal prediction}.
The difference between the original value 
$\mathbf{u}(0)$ and the final prediction  {$\bar{\mathbf{u}}(0)$}, namely, 
{$
    {\mathbb{E}[\mathbf{u}(0)] = \| \mathbf{u}(0) - \bar{\mathbf{u}}(0)\|},
$} 
is referred to as the \textit{multi-step reciprocal prediction error} in the Euclidean norm $\| \|$.
In this work, we have the following interesting finding: there is a very strong correlation between the network modeling error ${\mathcal E}(\mathbf{u}(0))$ and the multi-step reciprocal prediction error $\mathbb{E}(\mathbf{u}(0))$.
This allows us to use  $\mathbb{E}(\mathbf{u}(0))$ 
to approximate the desired network modeling error 
$  {\mathcal E}[\mathbf{u}(0)] = \|\mathcal{F}_\theta[\mathbf{u}(0)] - \mathbf{\Phi}_\Delta(\mathbf{u}(0))\|.$
Note that this multi-step reciprocal prediction error can be computed directly with the current state and the forward-backward evolution networks. Its computation does not require the ground-truth system state. 
Therefore, we can use reciprocal prediction error  to guide the selection of critical samples from regions with large modeling errors. These regions correspond to locations where the trained evolution network is likely to have high network modeling errors, indicating that the network needs additional information or better training in those areas. By targeting these critical regions for sample selection, we can focus on improving the network's performance with a minimal number of selected samples.

Let $J_{m}$ be the number of samples at $m$-th iteration of our critical sampling process, current sample set $\mathbf{\mathcal{S}}_F^{m}= \{[{\bf u}^j(0)\rightarrow {\bf u}^j(\Delta)] :  1\le j\le J_{m}\}$ is used to train the spatial-temporal evolution network. The corresponding forward-backward evolution networks are denoted by $\mathbf{\cal F}_\theta^{m}$ and $\mathbf{\cal G}_\vartheta^{m}$. 
Using the corresponding multi-step reciprocal prediction error distribution  $\mathbb{E}(\mathbf{u}(0))$, we can determine regions with high error values and collect a new set of samples
 $\mathbf{\Omega}^{m}$, which are added to the existing set of samples to update the training set: 
\begin{equation}
        \mathbf{\mathcal{S}}_F^{m+1} = 
        \mathbf{\mathcal{S}}_F^{m} \bigcup 
        \mathbf{\Omega}^m =
        \{[{\bf u}^j(0)\rightarrow {\bf u}^j(\Delta)] :  1\le j\le 
         J_{m+1}\}.
\end{equation}
The above sampling-learning process is repeated until the overall prediction error drops below the target threshold. According to Table \ref{tab-time} in Section \ref{sec:performance}, the complexity for training is increased by 4-5 times, which directly depends on the number of iterations needed to reach the threshold for the network modeling error. However, the above critical sample selection process can dramatically reduce the number of needed samples and related sample collection costs.
In the following sections, we will explain this process in more detail.

\subsection{Multi-Step Reciprocal Prediction Error and Critical Sampling}
\label{sec:error}
 In this section, we  show that there is a strong correlation between the multi-step reciprocal prediction error and the network modeling error of the forward temporal evolution network $\mathbf{\mathcal{F}}_\theta^m$.

\textbf{(1) Multi-step reciprocal prediction.} In our multi-step reciprocal prediction scheme, we have a forward temporal evolution network $ \mathbf{\mathcal{F}}_\theta^m$ and a backward evolution network $\mathbf{\mathcal{G}}_\vartheta^m$, which model the system evolution behaviors in the forward and backward time directions.
If the forward and backward evolution networks $\mathbf{\mathcal{F}}_\theta^m$ and $\mathbf{\mathcal{G}}_\vartheta^m$ are both well-trained, accurately approximating the forward and backward evolution operators, for an arbitrarily given system state 
$\mathbf{u}(0)$, the 
one-step reciprocal prediction error
$  
    {\mathbb{E}[\mathbf{u}(0)] = \| \mathbf{u}(0) - \bar{\mathbf{u}}(0)\|} = \|\mathbf{u}(0) - 
    {\mathbf{\mathcal{G}}_\vartheta^m}[
    {\mathbf{\mathcal{F}}_\theta^m}[\mathbf{u}(0)]] \|
$
should approach zero.
Now, we extend this one-step reciprocal prediction to $K$ steps.
As illustrated in Figure \ref{fig:idea},  starting from the initial condition 
$\mathbf{u}(0)$, we repeatedly apply the forward evolution network $\mathcal{F}_\theta^m$ to perform $K$-step prediction of the system future states,
$ 
    \hat{{\bf u}}(k\Delta)= 
    {\mathcal{F}_\theta^{m, (k)}}\left[{\bf u}(0)\right],
$
where $k=1, \cdots, K-1, K$, $\mathcal{F}_\theta^{m, (k)}$ represents the $k$-fold composition of $\mathcal{F}_\theta^m$:
\vspace{-3pt}
\begin{equation}
   \mathcal{F}_\theta^{m, (k)} = \underbrace { \mathcal{F}_\theta^m \circ \mathcal{F}_\theta^m \circ \dots \circ \mathcal{F}_\theta^m}_{k-{\rm fold}}.
\end{equation}
After $K$ steps of forward evolution prediction, then, starting with $ \hat{\bf u}(K\Delta)$, we perform $K$ steps of backward evolution prediction using  network $\mathcal{G}_\vartheta^m$:  
$
    \bar{\bf u}(k\Delta)=
    {\mathcal{G}_\vartheta^{m, (K-k)}}[\hat{\bf u}(K\Delta)], k=K-1, \cdots, 1, 0,
$
where
\vspace{-3pt}
\begin{equation}
   \mathcal{G}_\theta^{m, (K-k)} = \underbrace { \mathcal{G}_\theta^m \circ \mathcal{G}_\theta^m \circ \dots \circ \mathcal{G}_\theta^m}_{(K-k)-{\rm fold}}
\end{equation}
and reach back to time $t=0$. 
If the forward and backward evolution networks are both accurate, the forward prediction path and the backward prediction path should match each other. 
Motivated by  this, we define the multi-step reciprocal prediction error for the forward evolution network $\mathcal{F}_\theta^m$ as the deviation between the forward and backward prediction paths:
\vspace{-5pt}
\begin{equation}
\label{eq-lsf}
    \mathbb{E}[\mathbf{u}(0)] = \sum_{k=0}^K \Big\| \hat{\bf u}(k\Delta)-\bar{\bf u}(k\Delta)\Big\|^2.
    \vspace{-3pt}
\end{equation}
Note that, when computing  $\mathbb{E}[\mathbf{u}(0)]$, we only need the current system state $\mathbf{u}(0)$, the forward and backward evolution networks $\mathcal{F}_\theta^m$ and $\mathcal{G}_\vartheta^m$. 
Figure \ref{fig:reciprocal} shows several examples from the Damped Pendulum and 2D Nonlinear ODE systems listed in Table \ref{tab-eqs}. The top row shows examples with accurate predictions of their system states. We can see that their 
forward and backward prediction paths match well and the corresponding multi-step prediction error is very small. For comparison, the bottom shows examples with large prediction errors. 

\begin{figure*}[htbp!]
\begin{center}
\centerline{\includegraphics[width=0.965\linewidth]{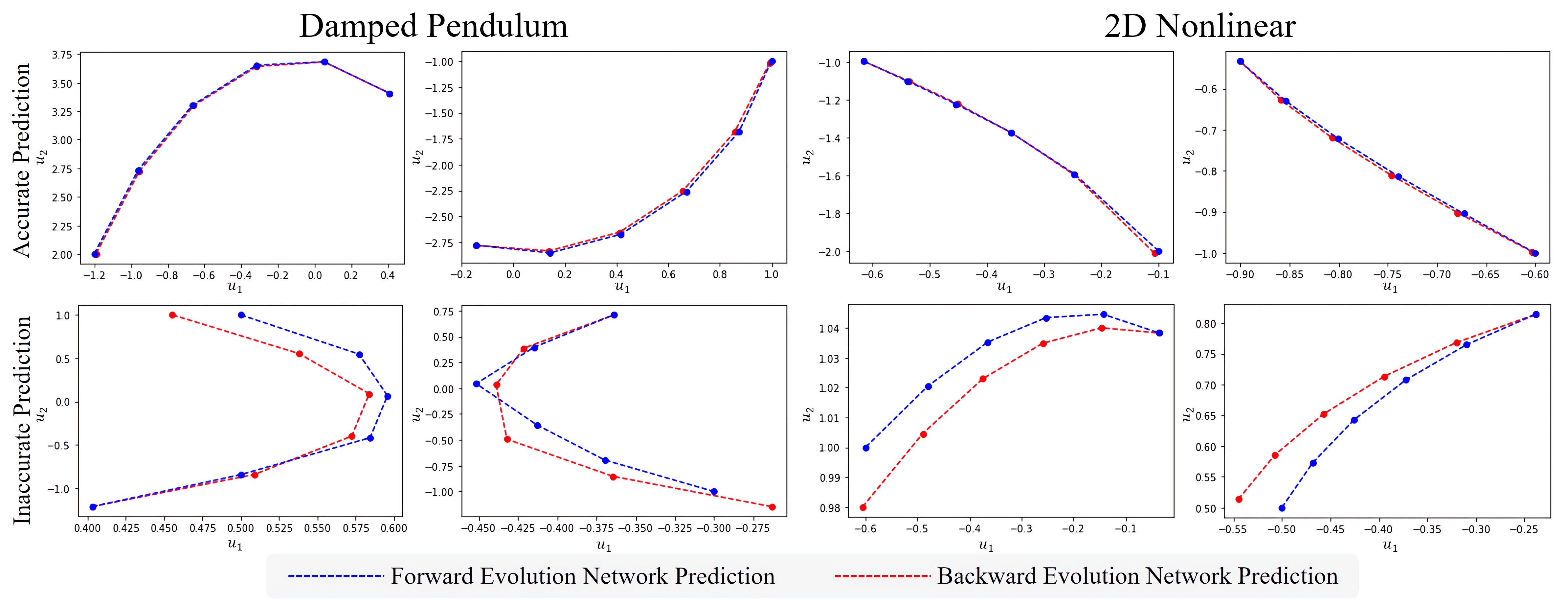}}
\vspace{-0.12in}
\caption{Examples of multi-step reciprocal prediction errors on Damped Pendulum and 2D Nonlinear ODE systems. After training the forward and backward evolution networks with 225 samples on the Damped Pendulum system and 467 samples on the 2D Nonlinear system, we take the locations with large/small multi-step reciprocal prediction error as the starting point in these example trajectories. Here, we take the reciprocal step as $K=5$.}
\label{fig:reciprocal}
\end{center}
% \vspace{-2pt}
\end{figure*}

\begin{figure*}[ht]
\vspace{-0.1in}
\begin{center}
\centerline{\includegraphics[width=0.965\linewidth]{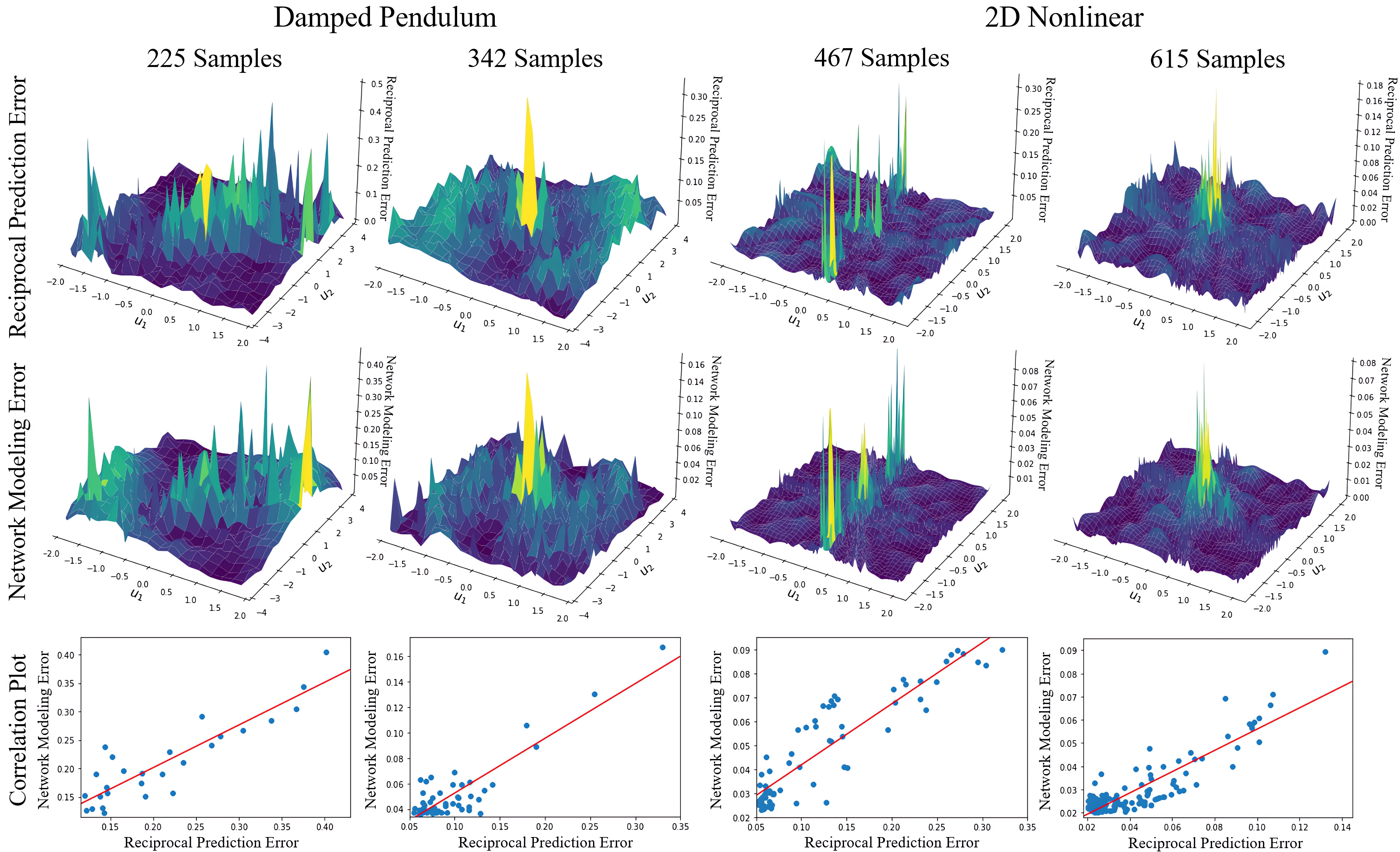}}
\vspace{-0.1in}
\caption{Correlation between network modeling error and multi-step reciprocal prediction error on Damped Pendulum and 2D Nonlinear ODE systems.}
\label{fig:correlation}
\end{center}
\vspace{-22pt}
\end{figure*}

\textbf{(2) Predicting the network modeling error.}
In this work, we find that there is a strong correlation between the network modeling error  ${\mathcal E}[\mathbf{u}(0)]$ and the multi-step reciprocal prediction error $\mathbb{E}[\mathbf{u}(0)]$. 
Figure \ref{fig:correlation} shows four examples of 
 $\mathbb{E}[\mathbf{u}(0)]$ (top row) and ${\mathcal E}[\mathbf{u}(0)]$ (middle row) for the Damped Pendulum and 2D Nonlinear system with different sizes of training samples. The bottom row shows the values of ${\mathcal E}[\mathbf{u}(0)]$ and $\mathbb{E}[\mathbf{u}(0)]$ of locations with large errors. We can see that there is a strong correlation between the network modeling error ${\mathcal E}[\mathbf{u}(0)]$ and the multi-step reciprocal prediction error $\mathbb{E}[\mathbf{u}(0)]$. 
 This correlation allows us to predict ${\mathcal E}[\mathbf{u}(0)]$ using $\mathbb{E}[\mathbf{u}(0)]$ which can be computed directly from the current system state without the need to know the ground-truth state.

\textbf{(3) Critical sampling and adaptive evolution operator learning.}
\label{sec:algorithm}
Once we are able to predict the network modeling error ${\mathcal E}[\mathbf{u}(0)]$ using the multi-step reciprocal prediction error $\mathbb{E}[\mathbf{u}(0)]$, we can develop a critical sampling and adaptive evolution learning algorithm. The central idea is  to select samples from locations with large values of error $\mathbb{E}[\mathbf{u}(0)]$ using the following iterative peak finding algorithm. 
Note that $\mathbf{u}(0)\in \mathbb{R}^n$. Write $\mathbf{u}(0)=[u_1, u_2, \cdots, u_n]$. 
Let ${\mathbf{\mathcal{S}}_F^{m}}=\{[{\bf u}^j(0)\rightarrow {\bf u}^j(\Delta)] :  1\le j\le {J_{m}}\} $ 
be the current sample set. 
To determine the locations of new samples, 
$\{{\bf u}^{j}(0)|{J_m}+1\le j\le {J_{m+1}}\}$, we find the peak value of multi-step reciprocal prediction error $\mathbb{E}[\mathbf{u}(0)]$ at every sampling point 
$\mathbf{u}(0)$ in the solution space $D$. In our experiment, we choose the sample point $\mathbf{u}(0)$ from the augmented sample set 
$\bar{S_F^m}$ defined in (\ref{eq:augmentation2}) of the following section.
The corresponding peak location is chosen to be $ {\bf u}^{J_m+1}(0)$ and the corresponding sample $[{\bf u}^{J_m+1}(0)\rightarrow {\bf u}^{J_m+1}(\Delta)]$ is collected. 
This process is repeated for $J_{m+1}-J_m$ times to collect $J_{m+1}-J_m$ samples in $\mathbf{\Omega}_m$, which is added to the current sample set:
\begin{equation}
        {\mathbf{\mathcal{S}}_F^{m+1}} \!= \!
        {\mathbf{\mathcal{S}}_F^{m}}
        \bigcup 
        {\mathbf{\Omega}^m} \!=\!
        \{[{\bf u}^j(0)\!\rightarrow \!{\bf u}^j(\Delta)]\! :  1\!\le\! j\!\le \!
        {J_{m+1}}\}.
\end{equation}

% \vspace{-20pt}

\subsection{Joint Spatial Dynamics and Temporal Evolution  Learning}
\label{sec:spatiotemporal}
\looseness = -1
Joint spatial dynamics and temporal evolution learning aims to couple local dynamics learning in the spatial domain and evolution learning in the temporal domain to achieve robust system evolution learning from the small set of selected critical samples. 

\looseness = -1
\textbf{(1) Sample augmentation based on local spatial dynamics.} Let 
${\mathbf{\mathcal{S}}_F^{m}} = \{[{\bf u}^j(0)\rightarrow {\bf u}^j(\Delta)] : 1\le j\le {J_m}\}
$ 
be the current set of samples collected from the dynamical system. In this work, $J_m$ is a small number. For example, in our experiments, $J_m$ is in the range of a few hundred.
From our experiments, we find that this learning process is unstable since the number of samples is too small. Our central idea is to design a spatial dynamics network to learn the local spatial dynamics so that we can predict or interpolate more samples at unknown locations from existing samples at known locations.
Specifically, let 
$\{{\bf v}^i(0) \in \mathbb{R}^n:  1\le i\le I\}$ be a large set of randomly selected points in the state space of $\mathbb{R}^n$. Let ${\bf v}^i(\Delta) = \mathbf{\Phi}_\Delta ({\bf v}^i(0))$ be the system's future state at  time $\Delta$ when its current state is ${\bf v}^i(0)$. Here, $\{{\bf v}^i(\Delta)\}$ are  predicted by the spatial dynamics network from the existing sample set $\mathbf{\mathcal{S}}_F^m$, denoted by
$\hat{\bf v}^i(\Delta)$.
They are added to 
$\mathbf{\mathcal{S}}_F^m$ as augmentation samples
\begin{equation}
    {\mathbf{\bar{\mathcal{S}}}_F^m} = 
    {\mathbf{\mathcal{S}}_F^m}
    \bigcup \mathbf{\mathcal{V}}_F,
    \label{eq:augmentation2}
\end{equation}
\begin{equation}
    \mathbf{\mathcal{V}}_F\! = \!\{[{\bf v}^i(0)\!\rightarrow\! \hat{\bf v}^i(\Delta)] : 1\!\le\! i\!\le\! I,\ 
    \hat{\bf v}^i(\Delta)\!=\!{\mathbf{\Gamma}_w^m}[{\mathbf{\mathcal{S}}_F^m}; {\bf v}^i(0)] \},
\label{eq:augmentation}
\end{equation}
where ${\mathbf{\Gamma}_w^m}[{\mathbf{\mathcal{S}}_F^m}; {\bf v}^i(0)] $ represent the spatial dynamics network at $m$-th iteration which predicts the future state of ${\bf v}^i(0)$ based local spatial change patterns using the existing samples $\mathbf{\mathcal{S}}_F^m$.
With this augmented sample set $\mathbf{\bar{\mathcal{S}}}_F^m$, we can train the temporal evolution network $\mathcal{F}_\theta^m$.

\textbf{(2) Learning the local spatial dynamics.} 
The dynamical system may exhibit highly nonlinear and complex behavior in the whole spatial domain, which could be challenging to be accurately modeled and predicted. However, within a small local neighborhood, its behavior will be much simpler and can be effectively learned by our spatial dynamics network $\mathbf{\Gamma}_w$. Specifically, 
given an arbitrary point 
${\bf v}(0)=[v_1, v_2, \cdots, v_n]$ in $\mathbb{R}^n$, we find its nearest $H$ points from the existing sample set 
$\mathbf{\mathcal{S}}_F^m$, and the corresponding samples are denoted by 
$
    \mathbf{\mathcal{S}_{{\bf v}(0)}} = \{[{\bf z}^h(0)\rightarrow {\bf z}^h(\Delta)] : 1\le h\le H\}
$, 
which are the input to the  spatial dynamics network.
We use a $p$th-order $n$-variate polynomial 
$\mathcal{P}(\mathbf{u})=\mathcal{P}[c_1, c_2, \cdots, c_P] (u_1, u_2, \cdots, u_n)$ to locally approximate the local spatial dynamics.
The coefficients of the polynomial are $[c_1, c_2, \cdots, c_P]$, which are predicted by the spatial dynamics network $\mathbf{\Gamma}_w$. For example, if $p=1$, this becomes a linear approximation with $P=n+1$ coefficients. If $p=2$, the number of coefficients, or the size of the network output becomes $P=\frac12 (n+2)(n+1)$. 
To summarize, the task of  the  spatial dynamics network $\mathbf{\Gamma}_w$ is to predict the coefficients of the polynomial $\mathcal{P}(\mathbf{u})$ from the set of $H$
neighboring samples $\mathbf{\mathcal{S}_{{\bf v}(0)}}$  so that 
$
    \mathcal{P}(\mathbf{u})|_{\mathbf{u}={\bf v}(0)} = {\bf v}(\Delta),
$
where ${\bf v}(\Delta)$ is the future state of the system at time $\Delta$ when its current state is 
${\bf v}(0)$, or ${\bf v}(\Delta) = \mathbf{\Phi}_\Delta ({\bf v}(0))$.

When training the spatial dynamics network $\mathbf{\Gamma}_w^m$, we can choose the sample from the existing sample set $\mathbf{\mathcal{S}}_F^m$ as the input 
${\bf v}(0)$ and the corresponding output as ${\bf v}(\Delta)$. The $L_2$ loss between {the predicted state} at time $\Delta$ and its true value for ${\bf v}(0)$, namely, 
\vspace{-3pt}
\begin{equation}
    \mathbb{L}_{SDN} = \sum_{{\bf v}(0)\in {\mathbf{\mathcal{S}}_F^m}} \Big\|{\bf v}(\Delta) -  {\mathbf{\Gamma}_w^m}[{\mathbf{\mathcal{S}}_F^m}; {\bf v}(0)] \Big\|^2.
\end{equation}

\textbf{(3) Joint learning of spatial dynamics and temporal evolution networks.} 
The temporal evolution network $\mathbf{\mathcal{F}}_\theta$ and the spatial dynamics network $\mathbf{\Gamma}_w$ aim to characterize the system behavior from two different perspectives, the temporal and spatial domains. 
In this work, we couple these two networks so that they can learn more effectively. 
Specifically, we use the spatial dynamics network $\mathbf{\Gamma}_w^m$ to generate a large set of samples
$\mathbf{\mathcal{V}}_F$, which is added to the existing samples 
$\mathbf{\mathcal{S}}_F^m$, as explained in (\ref{eq:augmentation}). 
This augmented sample set $\mathbf{\bar{\mathcal{S}}}_F^m$ is used to train the temporal evolution network $\mathbf{\mathcal{F}}_\theta^m$. Note that both networks 
are predicting the system future state from spatial and temporal domains. Therefore, we can introduce a consistency constraint between them. Specifically, let $\mathbf{\mathcal{Q}}=\{\mathbf{q}_l\in \mathbb{R}^n: 1\le l \le L \}$ be a set of randomly generated points in $\mathbb{R}^n$. 
We  use both networks to predict the future state at time $\Delta$ for each $\mathbf{q}_l$ at the initial state. The following consistency loss is used to train both networks:
\vspace{-3pt}
\begin{equation}
    \mathbb{L}_{C} = \sum_{l=1}^L \Big \|
    {\mathbf{\mathcal{F}}_\theta^m}
    [\mathbf{q}_l]- 
    {\mathbf{\Gamma}_w^m}
    [{\mathbf{\mathcal{S}}_F^m}; \mathbf{q}_l] \Big\|^2.
\end{equation}

\subsection{Theoretical Understanding}\label{sec:theoretical}
\looseness = -1
In this section, we provide some mathematical analysis results to understand and characterize the performance of the proposed critical sampling and adaptive evolution learning method, especially on the error bound of evolution operator learning. 

Let us consider the following autonomous ODE system as an example:
\begin{equation} \label{eq:ODE}
    \frac{d{\bf u}(t)}{dt} = {\mathcal H}({\bf u}(t)), \quad t \in \mathbb{R}^+,
\end{equation}
where ${\bf u}(t)\in \mathbb{R}^n$ are the state variables. 
Let $\mathbf{\Phi}_\Delta: \mathbb{R}^n\rightarrow \mathbb{R}^n$ be the {\em evolution operator}, which maps the system state from time $t=0$ to its next state at time $\Delta$: $    {\bf u}(t+\Delta) = \mathbf{\Phi}_\Delta({\bf u}(t)).$
It should be noted that, for autonomous systems, this evolution operator
$\mathbf{\Phi}_\Delta$ remains invariant for different time instance $t$. It only depends on the time difference $\Delta$.

\begin{lemma}\label{lem:backisforward}
For autonomous systems, the backward evolution operator $\mathbf{\Psi}_\Delta$ of system \eqref{eq:ODE} is actually the forward evolution operator of the following dynamical system 
	\begin{equation}\label{eq:backward-system}
		\frac{d}{dt}{\bf \bar u}(t) = - {\mathcal H}({\bf \bar u}(t)).
	\end{equation}
\end{lemma}

\begin{proof}
	Define ${\bf \bar u}(t) :={\bf u}(T-t)$ for an arbitrarily fixed $T>0$. It can be seen that 
	\begin{align*}
	\frac{d}{dt}{\bf \bar u}(t) &= \frac{d}{dt} ( {\bf u}(T-t) ) = - \left. \frac{d{\bf u}}{dt} \right|_{T-t} \\
	&
	= - {\mathcal H} ( {\bf u}(T-t) ) = - {\mathcal H}({\bf \bar u}(t)).
	\end{align*}
	This means ${\bf \bar u}(t)$ satisfies the ODEs \eqref{eq:backward-system}. The forward evolution operator of system \eqref{eq:backward-system}, which maps ${\bf \bar u}(t)$ to ${\bf \bar u}(t+\Delta)$, is equivalent to the mapping from 
	${\bf u}(T-t)$ to ${\bf u}(T-t-\Delta)$, which exactly coincides with the backward evolution operator $\mathbf{\Psi}_\Delta$  of system \eqref{eq:ODE}. The proof is completed. 
\end{proof}

It should be noted that the forward and backward evolution operators, denoted as $\mathbf{\Phi}_\Delta$ and $\mathbf{\Psi}_\Delta$ respectively,  are exact inverses of one another. In the following analysis, we assume $\mathcal H$ is Lipschitz continuous with Lipschitz constant $C_{\mathcal H}$ on a set $D \subset \mathbb R^n$. Here $D$ is a bounded region where we are interested in the solution behavior.

\begin{lemma}\label{lem:forwardPHI}
Define 
$$
\hat D_{\Delta} :=\{ {\bf u}\in D: \mathbf{\Phi}_t ( {\bf u} ) \in D~~ \forall t\in[0,\Delta]  \}.
$$ 
The forward evolution operator $\mathbf{\Phi}_\Delta$ of system \eqref{eq:ODE} is Lipschitz continuous on $\hat D_{\Delta}$, i.e., for any 
${\bf u}_1, {\bf u}_2 \in \hat D_{\Delta}$, 
\begin{equation}\label{eq:LP1}
	\left\| \mathbf{\Phi}_\Delta( {\bf u}_1 ) - \mathbf{\Phi}_\Delta( {\bf u}_2 )  \right\|_2 \le e^{ C_{\mathcal H} \Delta} \| {\bf u}_1- {\bf u}_2 \|_2.
\end{equation}
\end{lemma}

\begin{proof}
	This follows from a classical result in the dynamical system; see \cite{stuart1998dynamical}.
\end{proof}

\begin{lemma}\label{lem:backwardPSI}
	Define 
	$$
	\bar D_{\Delta} :=\{ {\bf u}\in D: \mathbf{\Psi}_t ( {\bf u} ) \in D~~ \forall t\in[0,\Delta]  \}.
	$$ 
	The backward evolution operator $\mathbf{\Psi}_\Delta$ of system \eqref{eq:ODE} is Lipschitz continuous on $\bar D_{\Delta}$, i.e., for any 
	${\bf u}_1, {\bf u}_2 \in \bar D_{\Delta}$, 
\begin{equation}\label{eq:LP2}
	\left\| \mathbf{\Psi}_\Delta( {\bf u}_1 ) - \mathbf{\Psi}_\Delta( {\bf u}_2 )  \right\|_2 \le e^{ C_{\mathcal H} \Delta} \| {\bf u}_1- {\bf u}_2 \|_2.
\end{equation}
\end{lemma}
\begin{proof}
According to Lemma \ref{lem:backisforward}, $\mathbf{\Psi}_\Delta$ is the forward evolution operator of system \eqref{eq:backward-system}. Following the idea of Lemma \ref{lem:forwardPHI} for $\mathbf{\Psi}_\Delta$ one can complete the proof. 
\end{proof}

We now derive a simple generic bound for the prediction error of our network model. Suppose the generalization error of the trained neural network is bounded: %. More specifically, we assume 
\begin{equation}\label{eq:FG_bound}
\begin{aligned}
&	\| {\mathcal F}_\theta - \mathbf{\Phi}_\Delta  \|_{L^\infty (D )} =: \epsilon_f < +\infty, \\
& \| {\mathcal G}_\vartheta - \mathbf{\Psi}_\Delta  \|_{L^\infty (D )} =: \epsilon_g < +\infty,
\end{aligned}
\end{equation}
where the supremum norm of a vector function $\mathbf{f}:\mathbb{R}^n\to\mathbb{R}^n$ over the  domain $D\subset\mathbb{R}^n$, denoted as $\| \mathbf{f} \|_{L^\infty (D )}$, is defined as
\begin{equation*}
    \| \mathbf{f} \|_{L^\infty (D )} = \underset{\mathbf{x}\in D}{\max} \; \| \mathbf{f}(\mathbf{x}) \|_2
\end{equation*}
with $\| \cdot \|_2$ denotes the vector 2-norm (i.e., Euclidean norm).  
Let $\hat {\bf u}^{(k)}$ be the forward predicted solution by the trained primal network model at time $t^{(k)} := t_0 + k \Delta$ starting from $t_0$, and $\bar {\bf u}^{(k)}$ be the backward predicted solution by the trained dual network model at time $t^{(k)}$ starting from $t^{(K)}$, where $0\le k\le K$. 
Denote the corresponding forward prediction error as $\hat {\mathcal E}^{(k)} := \| \hat {\bf u}^{(k)} -  {\bf u} ( t^{(k)} ) \|_2$ and the backward prediction error as $\bar {\mathcal E}^{(k)} := \| \bar {\bf u}^{(k)} -  {\bf u} ( t^{(k)} ) \|_2$, $k=0,1,\dots,K$. We then have the following estimates. 

\begin{theorem}\label{thm:main}
Suppose that the assumption \eqref{eq:FG_bound} holds, then we have: 
\begin{enumerate}
	\item If $\hat {\bf u}^{(k)}, {\bf u} ( t^{(k)} ) \in \hat D_\Delta$ for $0 \le k \le K-1$, then 
\begin{equation}\label{eq:estimate1}
		\hat {\mathcal E}^{(k)} \le \hat {\mathcal E}^{(0)} e^{ C_{\mathcal H} k \Delta  }  
	+  \left(  \frac{ e^{ C_{\mathcal H}  k\Delta  } - 1 }{ e^{ C_{\mathcal H} \Delta } - 1 } \right) \epsilon_f.
\end{equation}
	\item If $\bar {\bf u}^{(k)}, {\bf u} ( t^{(k)} ) \in \bar D_\Delta$ for $0 \le k \le K-1$, then 
\begin{equation}\label{eq:estimate2}
	\bar {\mathcal E}^{(k)} \le \bar {\mathcal E}^{(K)} e^{ C_{\mathcal H} (K-k) \Delta  }  
	+  \left(  \frac{ e^{ C_{\mathcal H}  (K-k) \Delta  } - 1 }{ e^{ C_{\mathcal H} \Delta } - 1 } \right) \epsilon_g.
\end{equation}
	\item
	Furthermore, if we take $\bar {\bf u}^{(0)} = {\bf u} ( t_0 )$ and pass $\bar {\bf u}^{(K)}$ as the input of the trained dual network with  $\bar {\bf u}^{(K)}=\hat {\bf u}^{(K)}$, then 
\begin{equation}\label{key0}
	\begin{aligned}
		\bar {\mathcal E}^{(k)} 
		& \le \min \Bigg\{   \left\| {\mathcal G}_\vartheta^{(K-k)} \circ {\mathcal F}_\theta^{(K-k)} - I \right\|_{L^\infty(D)} 
		\\
		& \qquad + \left(  \frac{ e^{ C_{\mathcal H}  k\Delta  } - 1 }{ e^{ C_{\mathcal H} \Delta } - 1 } \right) \epsilon_f,
		\\
		&
		\quad \left(  \frac{ e^{ C_{\mathcal H}  K\Delta  } - 1 }{ e^{ C_{\mathcal H} \Delta } - 1 } \right)
		\epsilon_f e^{ C_{\mathcal H} (K-k) \Delta  } 
		\\
		& \qquad 
		 + \left(  \frac{ e^{ C_{\mathcal H}  (K-k) \Delta  } - 1 }{ e^{ C_{\mathcal H} \Delta } - 1 } \right) \epsilon_g
		\Bigg\},
	\end{aligned}
\end{equation}	
and in particular when $k=0$, 	
\begin{equation}\label{key}
\begin{aligned}
		&	\| \bar {\bf u}^{(0)} -  {\bf u} ( t_0 ) \|_2 
		\\
	& \le \min \Bigg\{  \left\| {\mathcal G}_\vartheta^{(K)} \circ {\mathcal F}_\theta^{(K)} - I \right\|_{L^\infty(D)},
	\\
	&
	 \quad \left(  \frac{ e^{ C_{\mathcal H}  K\Delta  } - 1 }{ e^{ C_{\mathcal H} \Delta } - 1 } \right) \Big( \epsilon_f e^{ C_{\mathcal H} K \Delta  } 
 +  \epsilon_g \Big)
	\Bigg\},
\end{aligned}
\end{equation}
where $I$ denotes the identity map. 
\end{enumerate} 
\end{theorem}

\begin{proof}
Recall that $\hat {\bf u}^{(k)} = {\mathcal F}_\theta ( \hat {\bf u}^{(k-1)} )$ and 
${\bf u}({t^{(k)}}) = \mathbf{\Phi}_\Delta ( {\bf u}({t^{(k-1)}}) )$. Using the triangle inequality for the Euclidean norm, the assumption \eqref{eq:FG_bound}, and the Lipschitz continuity \eqref{eq:LP1} of $\mathbf{\Phi}_\Delta$ in Lemma \ref{lem:forwardPHI}, we can derive that 
\begin{align*}
 \hat {\mathcal E}^{(k)} & = \left\| \hat {\bf u}^{(k)} - {\bf u}({t^{(k)}}) \right\|_2
 \\
 & =  \left\| {\mathcal F}_\theta ( \hat {\bf u}^{(k-1)} ) -  \mathbf{\Phi}_\Delta ( {\bf u}({t^{(k-1)}}) ) \right\|_2 
 \\
 & \le \left\| {\mathcal F}_\theta ( \hat {\bf u}^{(k-1)} ) -  \mathbf{\Phi}_\Delta ( \hat {\bf u}^{(k-1)} ) \right\|_2 
 \\
& \qquad
 + \left\| \mathbf{\Phi}_\Delta ( \hat {\bf u}^{(k-1)} ) -  \mathbf{\Phi}_\Delta ( {\bf u}({t^{(k-1)}}) ) \right\|_2 
 \\
 & \le \| {\mathcal F}_\theta - \mathbf{\Phi}_\Delta  \|_{L^\infty (D )} 
 + e^{ C_{\mathcal H} \Delta}  \left\|  \hat {\bf u}^{(k-1)} -   {\bf u}({t^{(k-1)}})  \right\|_2
 \\
& =  e^{ C_{\mathcal H} \Delta}   \hat {\mathcal E}^{(k-1)} + \epsilon_f.
\end{align*}
Repeatedly utilizing such an estimate leads to 
\begin{align*}
	\hat {\mathcal E}^{(k)} & \le  e^{ C_{\mathcal H} \Delta}   \hat {\mathcal E}^{(k-1)} + \epsilon_f
	\\
	& \le e^{ 2 C_{\mathcal H} \Delta}   \hat {\mathcal E}^{(k-2)} + e^{ C_{\mathcal H} \Delta} \epsilon_f + \epsilon_f 
	\\
	& \le e^{ 3 C_{\mathcal H} \Delta}   \hat {\mathcal E}^{(k-3)} + e^{ 2 C_{\mathcal H} \Delta} \epsilon_f^2 + e^{ C_{\mathcal H} \Delta} \epsilon_f   + \epsilon_f
	\\ 
	& \le \cdots
	\\
	& \le e^{ k C_{\mathcal H} \Delta}   \hat {\mathcal E}^{(0)} + \epsilon_f \sum_{j=0}^{k-1} e^{ j C_{\mathcal H} \Delta},
\end{align*}
which completes the proof of \eqref{eq:estimate1}. Similarly, for the backward prediction procedure, we recall that  
$\bar {\bf u}^{(k)} = {\mathcal G}_\vartheta ( \bar {\bf u}^{(k+1)} )$ and 
${\bf u}({t^{(k)}}) = \mathbf{\Psi}_\Delta ( {\bf u}({t^{(k+1)}}) )$, and then use Lemma \ref{lem:backwardPSI} to deduce that 
\begin{align*}
	\bar {\mathcal E}^{(k)} & = \left\| \bar {\bf u}^{(k)} - {\bf u}({t^{(k)}}) \right\|_2
	\\
	& =  \left\| {\mathcal G}_\vartheta ( \bar {\bf u}^{(k+1)} ) -  \mathbf{\Psi}_\Delta ( {\bf u}({t^{(k+1)}}) ) \right\|_2 
	\\
	& \le \left\| {\mathcal G}_\vartheta ( \bar {\bf u}^{(k+1)} ) -  \mathbf{\Psi}_\Delta ( \bar {\bf u}^{(k+1)} ) \right\|_2 
	\\
    & \qquad
	+ \left\| \mathbf{\Psi}_\Delta ( \bar {\bf u}^{(k+1)} ) -  \mathbf{\Psi}_\Delta ( {\bf u}({t^{(k+1)}}) ) \right\|_2 
	\\
	& \le \| {\mathcal G}_\vartheta - \mathbf{\Psi}_\Delta  \|_{L^\infty (D )}  + e^{ C_{\mathcal H} \Delta}  \left\|  \bar {\bf u}^{(k+1)} -   {\bf u}({t^{(k+1)}})  \right\|_2  
	\\
    & =  e^{ C_{\mathcal H} \Delta}   \bar {\mathcal E}^{(k+1)} + \epsilon_g.
\end{align*}
Repeatedly utilizing such an estimate leads to 
\begin{align*}
	\bar {\mathcal E}^{(k)} & \le  e^{ C_{\mathcal H} \Delta}   \bar {\mathcal E}^{(k+1)} + \epsilon_g
	\\
	& \le e^{ 2 C_{\mathcal H} \Delta}   \bar {\mathcal E}^{(k+2)} + e^{ C_{\mathcal H} \Delta} \epsilon_g + \epsilon_g 
	\\
	& \le e^{ 3 C_{\mathcal H} \Delta}   \bar {\mathcal E}^{(k+3)} + e^{ 2 C_{\mathcal H} \Delta} \epsilon_g^2 + e^{ C_{\mathcal H} \Delta} \epsilon_g   + \epsilon_g
	\\ 
	& \le \cdots
	\\
	& \le e^{ (K-k) C_{\mathcal H} \Delta}   \bar {\mathcal E}^{(K)} + \epsilon_g \sum_{j=0}^{K-k-1} e^{ j C_{\mathcal H} \Delta},
\end{align*}
which completes the proof of \eqref{eq:estimate2}. Furthermore, if we take $\bar {\bf u}^{(0)} = {\bf u} ( t_0 )$ and pass $\bar {\bf u}^{(K)}$ as the input of the trained dual network with  $\bar {\bf u}^{(K)}=\hat {\bf u}^{(K)}$, then $\hat {\mathcal E}^{(0)}=0$ and  $\bar {\mathcal E}^{(K)}=\hat {\mathcal E}^{(K)}$. Combining \eqref{eq:estimate1} and \eqref{eq:estimate2} gives 
\begin{align} \nonumber
\bar {\mathcal E}^{(k)} & \le \hat {\mathcal E}^{(K)} e^{ C_{\mathcal H} (K-k) \Delta  }  
+  \left(  \frac{ e^{ C_{\mathcal H}  (K-k) \Delta  } - 1 }{ e^{ C_{\mathcal H} \Delta } - 1 } \right) \epsilon_g
\\\nonumber
& \le \left(  \frac{ e^{ C_{\mathcal H}  K\Delta  } - 1 }{ e^{ C_{\mathcal H} \Delta } - 1 } \right)
\epsilon_f e^{ C_{\mathcal H} (K-k) \Delta  } 
\\ 
& \qquad
+ \left(  \frac{ e^{ C_{\mathcal H}  (K-k) \Delta  } - 1 }{ e^{ C_{\mathcal H} \Delta } - 1 } \right) \epsilon_g.
\label{eq:003}
\end{align}
On the other hand, we observe that 
\begin{align} \nonumber
	\bar {\mathcal E}^{(k)} & = \left\| \bar {\bf u}^{(k)} - {\bf u}({t^{(k)}}) \right\|_2
	\\ \nonumber
	& \le \left\| \bar {\bf u}^{(k)} - \hat {\bf u}^{(k)} \right\|_2 + \left\| \hat {\bf u}^{(k)} - {\bf u}({t^{(k)}}) \right\|_2 
	\\ \nonumber
	& = \left\| {\mathcal G}_\vartheta^{(K-k)} \circ {\mathcal F}_\theta^{(K-k)} \hat {\bf u}^{(k)}  - \hat {\bf u}^{(k)} \right\|_2 +  \hat {\mathcal E}^{(k)}
	\\ \nonumber
	&\le  \left\| {\mathcal G}_\vartheta^{(K-k)} \circ {\mathcal F}_\theta^{(K-k)} - I \right\|_{L^\infty(D)} 
	+ \hat {\mathcal E}^{(k)}
	\\ \label{eq:002}
	&\le \left\| {\mathcal G}_\vartheta^{(K-k)} \circ {\mathcal F}_\theta^{(K-k)} - I \right\|_{L^\infty(D)} + \left(  \frac{ e^{ C_{\mathcal H}  k\Delta  } - 1 }{ e^{ C_{\mathcal H} \Delta } - 1 } \right) \epsilon_f, 
\end{align}
where we have used the triangular inequity for the Euclidean norm and the estimate \eqref{eq:estimate1}. Here the second inequality in \eqref{eq:002} follows from 
\begin{equation*}
\begin{aligned}
    &\left\| {\mathcal G}_\vartheta^{(K-k)} \circ {\mathcal F}_\theta^{(K-k)} \hat {\bf u}^{(k)} - \hat {\bf u}^{(k)} \right\|_2 \\
    &\qquad\qquad\le \underset{{\bf u} \in \hat D_\Delta}{\max} \; \left\| {\mathcal G}_\vartheta^{(K-k)} \circ {\mathcal F}_\theta^{(K-k)}  {\bf u} - {\bf u} \right\|_2 \\
    &\qquad\qquad= \left\| {\mathcal G}_\vartheta^{(K-k)} \circ {\mathcal F}_\theta^{(K-k)} - I \right\|_{L^\infty(\hat D_\Delta)} \\
    &\qquad\qquad\le \left\| {\mathcal G}_\vartheta^{(K-k)} \circ {\mathcal F}_\theta^{(K-k)} - I \right\|_{L^\infty(D)}.
\end{aligned}
\end{equation*}
Combining \eqref{eq:002} with \eqref{eq:003} gives \eqref{key0} and completes the proof. 
\end{proof}

The analysis suggests that the reciprocal prediction error is correlated with the network modeling error, which provides theoretical support for our finding. 
Critical sampling can help to effectively reduce the reciprocal prediction error and suppress the undesirable error growth, thereby enhancing the accuracy and stability of our model.

\section{Experimental Results}
\label{sec-results}
In this section, we present performance results obtained by our proposed method and baseline method on four dynamical systems to demonstrate the effectiveness of our method.

\subsection{Experimental Settings}
We follow the evaluation procedure used in existing research, for example, those reviewed in Section \ref{sec-related}, to evaluate the performance of our proposed method on specific examples of dynamical systems. We consider four representative systems with ODEs and PDEs as their governing equations, as summarized in  Table \ref{tab-eqs}. 
They include (1) the Damped Pendulum ODE equations in $\mathbb{R}^2$, (2) a nonlinear ODE equation in $\mathbb{R}^2$, (3) the Lorenz system (ODE) in $\mathbb{R}^3$, and (4) the Viscous Burgers' equation (PDE). 
Note that, for the final PDE system, we approximate it in a generalized Fourier space to reduce the problem to finite dimensions as in \cite{Wu_2020_JCP}. We use the projection operator $\mathcal{P}_n: \mathbb{V} \rightarrow \mathbb{V}_n$, where $\mathbb{V}_n=\mathrm{span}\left\{ \sin \left( jx \right): 1 \leqslant j \leqslant n\right\}$ with $n=9$.
Certainly, our proposed method can be also applied to many other dynamical systems, we simply use these four example systems to demonstrate the performance of our new method. 
In Section \ref{subsec-config}, we provide detailed descriptions of how to obtain the training samples for the dynamical systems. 
{An overview and pseudo-code of our method can be found in Section \ref{subsec-algorithm}.}

\begin{table*}[ht]
\caption{Overview of the 4 governing equation systems we demonstrate in this work.}
\vspace{-10pt}
\label{tab-eqs}
\begin{center}
\begin{small}
\begin{sc}
\resizebox{0.8\linewidth}{!}{
\begin{tabular}{ll}
\toprule
System & Governing Equations\\
\hline\vspace{-5pt}\\\
Damped Pendulum Equation & $    
    \begin{cases}
    \frac{d}{dt}{u}_1=u_2,\\
    \frac{d}{dt} {u}_2=-0.2 u_2-8.91 \sin u_1.
    \end{cases}$\vspace{7pt}\\
    \hline\vspace{-5pt}\\
A 2D Nonlinear Equation  & $    
    \begin{cases}
    \frac{d}{dt} {u}_1=u_2-u_1\left( {u_1}^2+{u_2}^2-1\right),\\
    \frac{d}{dt} {u}_2=-u_1-u_2\left( {u_1}^2+{u_2}^2-1\right).
    \end{cases}$\vspace{7pt}\\
\hline\vspace{-5pt}\\
Lorenz System  & $    
    \begin{cases}
    \frac{d}{dt} {u}_1=10\left( u_2 - u_1 \right),\\
    \frac{d}{dt} {u}_2=u_1\left( 28-u_3\right)-u_2,\\
    \frac{d}{dt} {u}_3=u_1u_2-(8/3)u_3.
    \end{cases}$\vspace{7pt}\\
\hline\vspace{-5pt}\\
Viscous Burgers’ Equation & $        \left\{
    \begin{array}{ll}
    u_t+\left( \frac{u^2}{2} \right)_x=0.1 u_{xx}, & (x, t)\in \left( -\pi, \pi\right) \times \mathbb{R}^+,\\
    u(-\pi, t) = u(\pi, t) = 0 , & t \in \mathbb{R}^+.
    \end{array}
    \right.$\vspace{5pt}\\
\bottomrule
\end{tabular}
}
\vspace{-10pt}
\end{sc}
\end{small}
\end{center}

\end{table*}

\subsection{System Configurations} 
\label{subsec-config}
For the ODE examples, we follow the procedure in \cite{qin2019data} to generate the training data pairs
$\{[{\bf u}^j(0), {\bf u}^j(\Delta)]\}$ as follows. First, we generate $J$ system state vectors $\{{\bf u}^j(0)\}_{j=1}^J$ at time 0 based on a uniform distribution over a computational domain $D$. Here, $D$ is the region where we are interested in the solution space. It is typically chosen to be a hypercube prior to computation, which will be explained in the following. Then, for each $j$, starting from
${\bf u}^j(0)$, we solve the true ODEs for a time lag of $\Delta$ using a highly accurate  ODE solver to generate ${\bf u}^j(\Delta)$. {\em Notice that, once the data is generated, we assume that the true equations are unknown, and the sampled data pairs are the only known information during the learning process.}

For the first example dynamical system listed in Table \ref{tab-eqs}, its computational domain is $D=\left[ -\pi, \pi\right] \times \left[ -2\pi, 2\pi\right]$. We choose $\Delta = 0.1$.
For the second system, the computational domain is $D=\left[ -2, 2\right]^2$. The time lag $\Delta$ is set as $0.1$. 
For the third system, the computational domain is $D=\left[ -25, 25\right]^2 \times \left[ 0, 50\right]$. The time lag $\Delta$ is set as $0.01$. 

For the Viscous Burgers' PDE system, because the evolution operator is defined between infinite-dimensional spaces, and we  approximate it in a modal space, namely, a generalized Fourier space, in
order to reduce the problem to finite dimensions as in \cite{Wu_2020_JCP}. 
We follow the same procedure specified in \cite{Wu_2020_JCP} to choose a basis of the finite-dimensional space $\mathbb{V}_n$ to represent the solutions, then apply the projection operator to  project the snapshot data to $\mathbb{V}_n$ to obtain the training data in the generalized Fourier space. The choice of basis functions is fairly flexible, any basis suitable for spatial approximation of the solution data can be used. Once the basis functions are selected, a projection operator $\mathcal{P}_n: \mathbb{V} \rightarrow \mathbb{V}_n$ is applied to obtain the solution in the finite-dimensional form. 
The approximation space is chosen to be relatively larger as $\mathbb{V}_n=\mathrm{span}\left\{ \sin \left( jx \right): 1 \leqslant j \leqslant n\right\}$ with $n=9$. The time lag $\Delta$ is taken as $0.05$. The domain $D$ in the modal space is set as $\left[ -1.5, 1.5\right] \times \left[ -0.5, 0.5 \right] \times \left[ -0.2, 0.2 \right]^2 \times \left[ -0.1, 0.1 \right]^2 \times \left[ -0.05, 0.05 \right]^2 \times \left[ -0.02, 0.02 \right]$, from which we sample the training samples.

Our task is to demonstrate how our proposed method is able to significantly reduce the number of samples needed for evolution learning. Specifically, for the baseline method \cite{qin2019data,Wu_2020_JCP}, we select locations based on a uniform distribution in the solution space to collect samples for evolution operator learning. For example, for the first dynamical system, Damped Pendulum system (ODE) in a 2-dimensional space, the baseline method uses 14,400 samples to achieve an average network modeling error of  0.026. We then use our method to adaptively discover critical samples and refine the evolution network to reach the same or even smaller network modeling error. We demonstrate that, to achieve the same modeling error, our proposed method needs much fewer samples.

\subsection{Implementation Details}
\label{subsec-details}
In all examples, we use the {recursive} ResNet (RS-ResNet) architecture in \cite{he2016deep,qin2019data}, which is a block variant of the ResNet and has been proven in \cite{qin2019data,Wu_2020_JCP} to be highly suitable for learning flow maps and evolution operators.

For all  four systems, the batch size is set as 10.  In the two 2-dimensional ODE systems, we use the one-block ResNet method with each block containing 3 hidden layers of equal width of 20 neurons, while in the 3-dimensional ODE system, we use the one-block ResNet method with each block containing 3 hidden layers of equal width of 30 neurons. For the final PDE system, we use the four-block ResNet method with each block containing 3 hidden layers of equal width of 20 neurons. Adam optimizer with betas equal $(0.9, 0.99)$  is used for training. In the two 2-dimensional ODE systems, all the networks are trained with 150 epochs. In the Lorenz system and Viscous Burgers' equation,  all the networks are trained with 60 epochs. The initial learning rate is set as $10^{-3}$ , and will decay gradually to $10^{-6}$ during the training process. The number of reciprocal steps $K$ is set as 5. All  networks are trained using PyTorch on one single RTX 3060 GPU.

In the four example systems, we evaluate the performance of our models on time duration $t \in [0, 20]$, $t \in [0, 10]$, $t \in [0, 5]$, $t \in [0, 2]$, respectively.
For the first two ODE systems, the network modeling error is evaluated by average MSE error at each time step on 50 different arbitrarily chosen solution trajectories. For the Lorenz system, we evaluate the network by average MSE error at each time step under 50 different initial conditions. For the final PDE system, the network modeling error is evaluated by the average $L_2$ norm error on 100 points at time $t=2$ under 50 different initial conditions.

\renewcommand{\arraystretch}{1.23}
\begin{table*}[ht]
\caption{Samples for learning the system evolution using the baseline method and our method.}
\vspace{-5pt}
\label{tab:samples}
\vspace{-0.05in}
\begin{center}
\begin{small}
\resizebox{0.85\linewidth}{!}{
\begin{tabular}{l|cc|cc|c}
\hline
\multirow{2}{*}{Dynamical System} & \multicolumn{2}{c|}{Baseline} & \multicolumn{2}{c|}{Our Work} & \multirow{2}{*}{ Ratio}  \\ \cline{2-5}
 & Samples & Prediction Error & Samples & Prediction Error & \\ \hline
 
Damped Pendulum  & 14400  &  0.02630 $\pm$ 0.01200 & \textbf{417} & \textbf{0.02411} $\pm$ 0.00991 & 34.53\\ \hline

2D Nonlinear  & 14400  & 0.00037 $\pm$ 0.00021 & \textbf{925} & \textbf{0.00035} $\pm$ 0.00015 & 15.57\\ \hline

Lorenz System  & 1000000  & 0.19685 $\pm$ 0.07768 & \textbf{1765} & \textbf{0.19357} $\pm$ 0.05695 &  566.57\\ \hline

Viscous Burgers' Eq.  & 500000  & 0.01679 $\pm$ 0.00878 & \textbf{19683} & \textbf{0.01652} $\pm$ 0.00818 & 25.40\\ \hline

\end{tabular}
}
\vspace{-0.15in}
\end{small}
\end{center}
\end{table*}

\subsection{Pseudo-Code and Overview of Our Proposed Method}
\label{subsec-algorithm}
Our proposed method of critical sampling and adaptive evolution operator learning algorithm is summarized in Algorithm \ref{alg-critical}. Our proposed method has the following steps. First, we generate $J_{m}$ data pairs based on a uniform distribution over a computational domain $D$ using a highly accurate ODE/PDE solver. Then, we train the spatial dynamics network $\mathbf{\Gamma}_w^m$ using the generated data pairs and use the network to generate a large set of additional samples. After that, we train forward evolution network $\mathbf{\mathcal{F}}_\theta^m$ using the data pairs, and train backward evolution network $\mathbf{\mathcal{G}}_\vartheta^m$ using the reversed data pairs. Multi-step reciprocal prediction errors are evaluated on different locations in the computational domain $D$ with both networks. Finally, we collect the samples from the locations with peak reciprocal prediction errors. Those samples should be added to the initial set and  the whole process is repeated until the network modeling error ${\mathcal E}[\mathbf{u}(0)]$ is  smaller than the threshold.

\SetKwRepeat{Do}{do}{while}
\begin{algorithm}[ht]
\label{alg-critical}
\begin{normalsize}
	\caption{Critical Sampling and Adaptive Evolution Operator Learning Algorithm}
	\LinesNumbered
	\KwIn{Number of samples in the initial set $J_{m}$; Number of samples in the updated set $J_{m+1}$; Training hyper-parameters.}
	\KwOut{Optimized forward evolution network $\mathbf{\mathcal{F}}_\theta$.}
	Generate $J_{m}$ data pairs based on a uniform distribution over a computational domain $D$, initialize the sample set $\mathbf{\mathcal{S}}_F^{m}= \{[{\bf u}^j(0)\rightarrow {\bf u}^j(\Delta)] :  1\le j\le J_{m}\}$\;
	\Repeat{network modeling error  ${\mathcal E}[\mathbf{u}(0)]$ smaller than threshold}{
	    \tcp{Sample augmentation based on local spatial dynamics.}
		Train spatial dynamics network $\mathbf{\Gamma}_w^m$ using $\mathbf{\mathcal{S}}_F^m$\; 
    	Use $\mathbf{\Gamma}_w$ to generate a large set of samples $\mathbf{\mathcal{V}}_F$, add to the existing sample set: $ \mathbf{\bar{\mathcal{S}}}_F^m = \mathbf{\mathcal{S}}_F^m \bigcup \mathbf{\mathcal{V}}_F$\;
    	\vspace{8pt}
    	\tcp{Multi-step reciprocal prediction.}
    	Reverse the data pairs in $\mathbf{\bar{\mathcal{S}}}_F^m$ to get $\mathbf{\bar{\mathcal{S}}}_G^m$\;
    	Train forward evolution network $\mathbf{\mathcal{F}}_\theta^m$ using $\mathbf{\bar{\mathcal{S}}}_F^m$, backward evolution network $\mathbf{\mathcal{G}}_\vartheta^m$ using $\mathbf{\bar{\mathcal{S}}}_G^m$\;
    	Perform $K$-step forward prediction using $\mathbf{\mathcal{F}}_\theta^m$ to get $\hat{{\bf u}}(k\Delta)$\;
    	Perform $K$-step backward prediction using $\mathbf{\mathcal{G}}_\vartheta^m$ to get $\bar{\bf u}(k\Delta)$\;
    	Calculate multi-step reciprocal prediction error $\mathbb{E}[\mathbf{u}(0)] = \sum_{k=0}^K \big\| \hat{\bf u}(k\Delta)-\bar{\bf u}(k\Delta)\big\|^2$\;
    	\vspace{5pt}
    	\tcp{Critical sampling.}
    	Initialize critical sample set $\mathbf{\Omega}_m$ to empty set\;
    	\For{i \textbf{in} \textnormal{\{1, 2, $\dots$, $J_{m+1}-J_{m}$\}}}{
    	    Choose location with peak reciprocal prediction error $\mathbb{E}[\mathbf{u}(0)]$ to be ${\bf u}^{J_{m}+i}(0)$\;
    	    Collect corresponding sample $[{\bf u}^{J_m+i}(0)\rightarrow {\bf u}^{J_m+i}(\Delta)]$, add to $\mathbf{\Omega}_m$.
    	}
    	Add the critical sample set to the current sample set (without augmented samples $\mathbf{\mathcal{V}}_F$);
            $\mathbf{\mathcal{S}}_F^{m+1} =\mathbf{\mathcal{S}}_F^m \bigcup \mathbf{\Omega}_m$\;\vspace{8pt}
    	$\mathbf{\mathcal{S}}_F^{m}=\mathbf{\mathcal{S}}_F^{m+1}$, $J_{m}=J_{m+1}$\;}
\end{normalsize}
\end{algorithm}

\subsection{Performance Results}
\label{sec:performance}
We choose the evolution learning method developed in \cite{qin2019data,Wu_2020_JCP} as our baseline. This method has achieved impressive performance in learning the evolution behaviors of autonomous systems and attracted much attention from the research community. On top of this method, we implement our
proposed method of critical sampling and adaptive evolution learning. We demonstrate that, to achieve the same modeling error, our method needs much fewer samples.

\begin{figure*}[!htbp]
  \centering
    \includegraphics[width=0.85\linewidth]{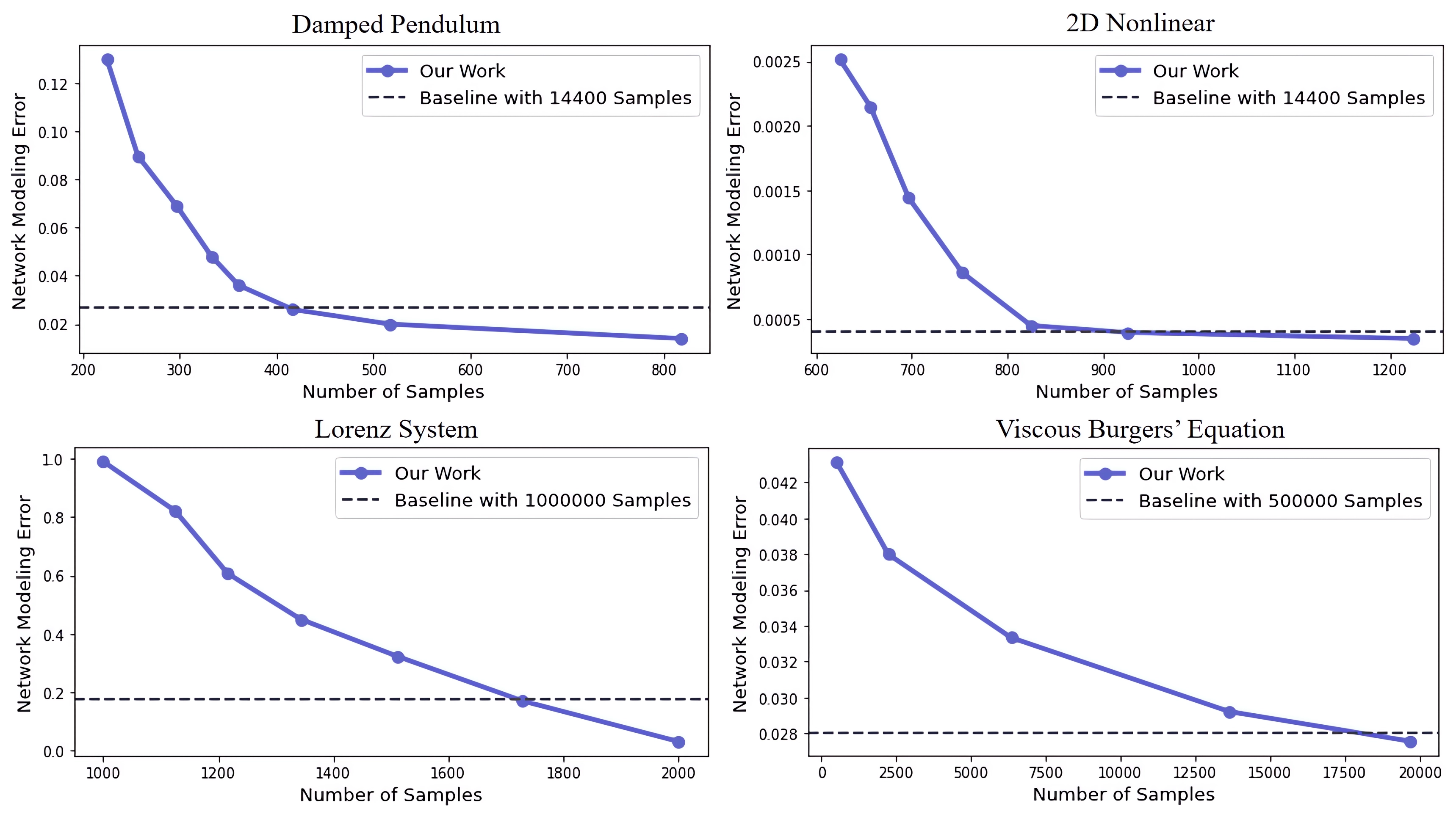}
    \vspace{-0.15in}
  \caption{The critical sampling and adaptive learning results on four dynamical systems.}
    \label{fig:all-error-samples}
    \vspace{-12pt}
\end{figure*}

\begin{figure*}[htbp!]
\begin{center}
\centerline{\includegraphics[width=\linewidth]{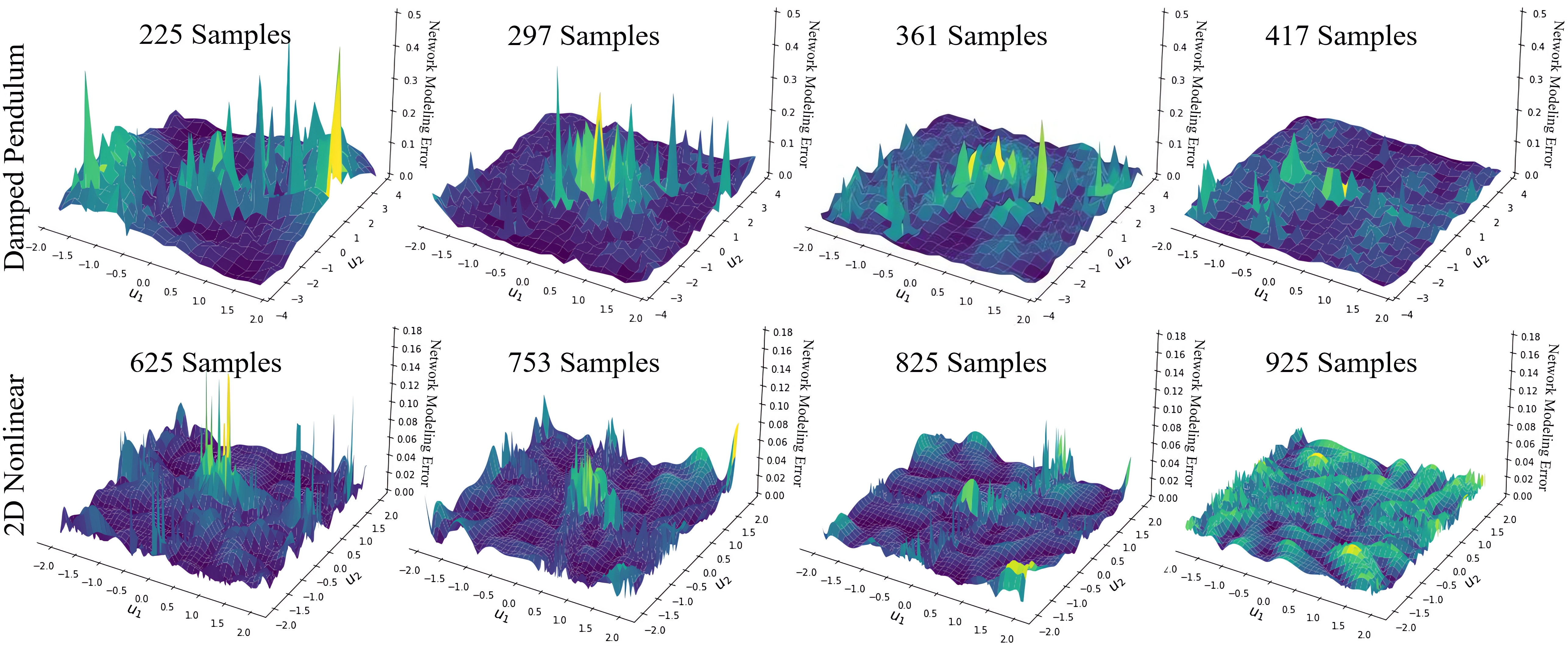}}
\vskip -0.1in
\caption{The reduction of network modeling error with more and more critical samples are collected for the Damped Pendulum system (top) and the 2D Nonlinear system (bottom).}
\label{fig:ODE-error}
\end{center}
\vspace{-0.15in}
\end{figure*}

\begin{figure*}[htbp!]
\begin{center}
\centerline{\includegraphics[width=\linewidth]{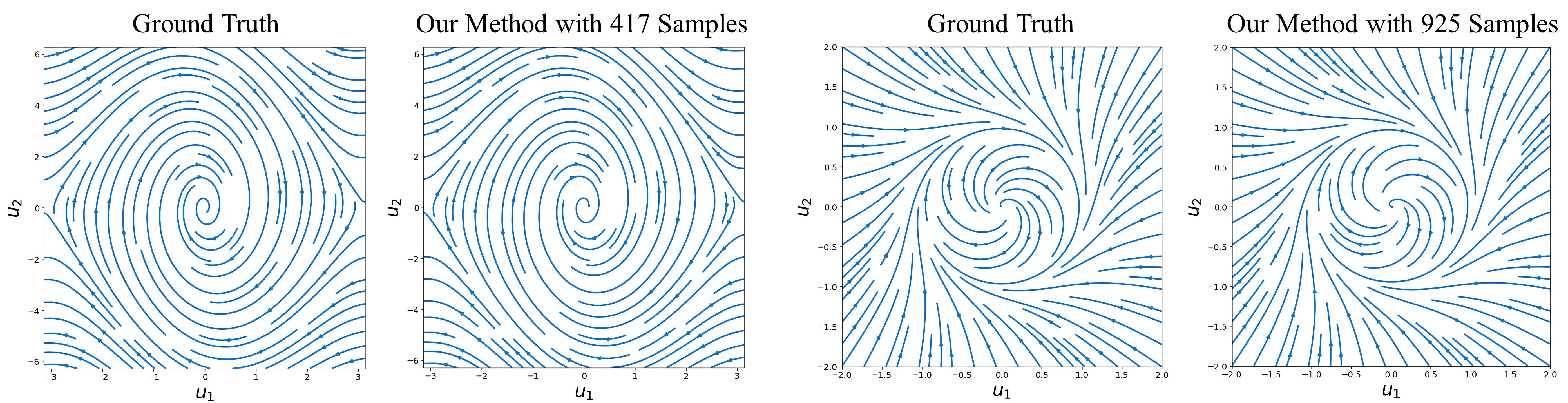}}
\vskip -0.1in
\caption{Phase portraits of the solutions obtained by critical sampling and adaptive evolution learning method for the Damped Pendulum system (left) and 2D Nonlinear system (right).}
\label{fig:solution}
\end{center}
\end{figure*}

\begin{figure*}[htbp!]
\begin{center}
\centerline{\includegraphics[width=0.8\linewidth]{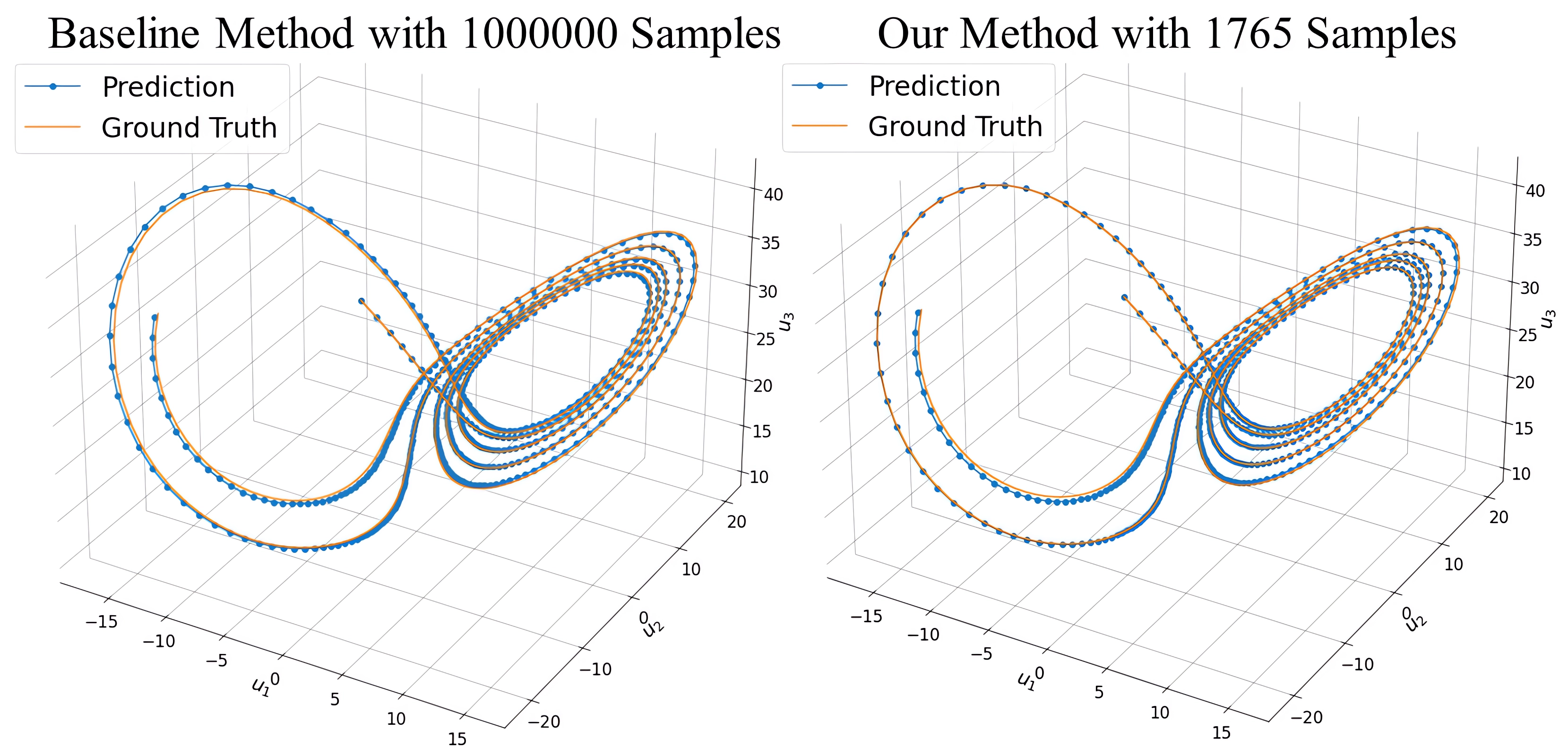}}
\vskip -0.1in
\caption{Solutions obtained by critical sampling and adaptive evolution learning method for the Lorenz system.}
\label{fig:ode3solution}
\end{center}
\vspace{-10pt}
\end{figure*}

\begin{figure*}[htbp!]
\begin{center}
\centerline{\includegraphics[width=1.0\linewidth]{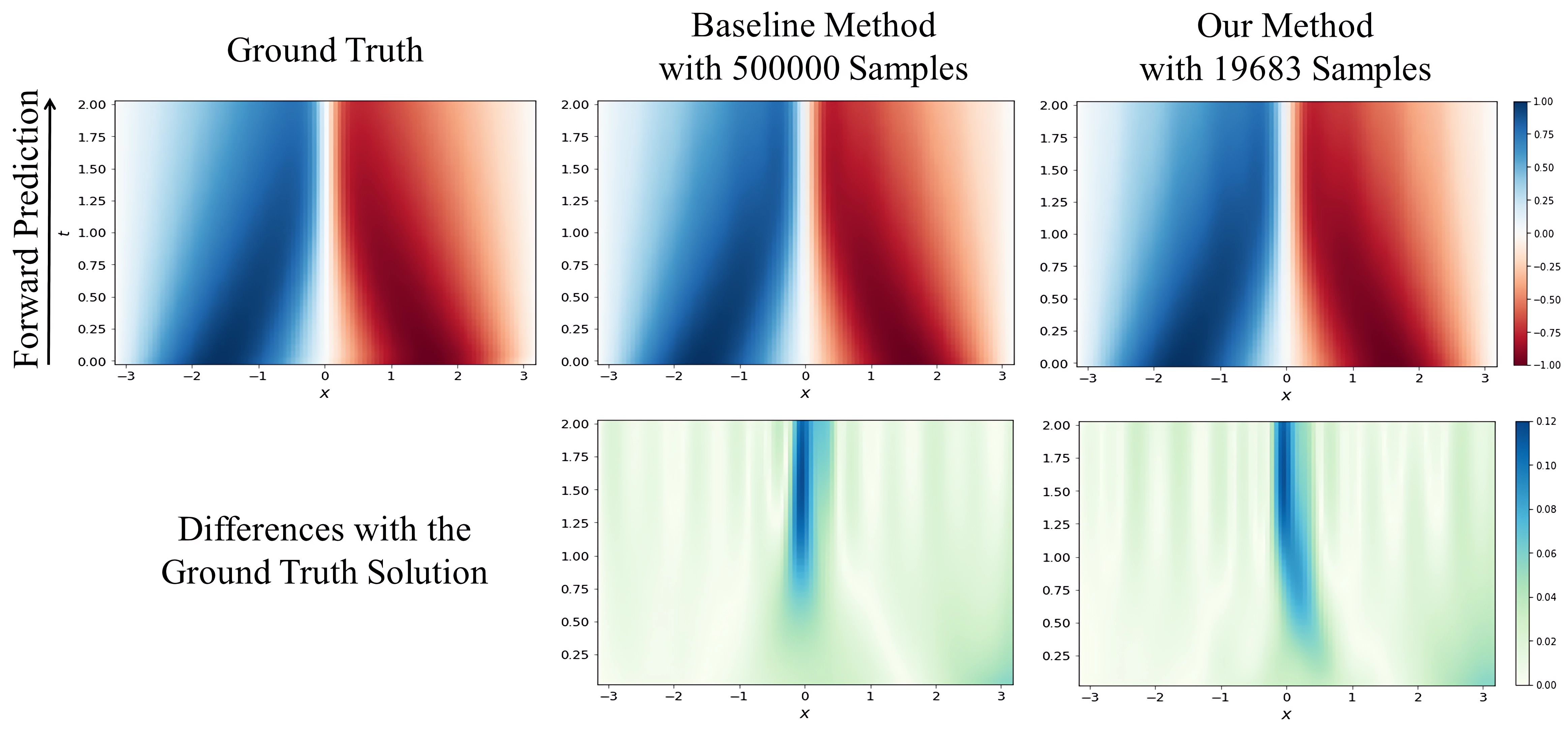}}
\vskip -0.1in
\caption{Solutions obtained by critical sampling and adaptive evolution learning method for the Viscous Burgers' equation.}
\label{fig:pdesolution}
\end{center}
\end{figure*}

Table \ref{tab:samples} compares the numbers of samples needed for learning the system evolution by the baseline method and our critical sampling and adaptive learning method. The prediction errors are evaluated on 50 arbitrarily chosen solution trajectories in the computational domain. All testing states are not included in the training set. Average errors and standard deviations are reported for each dynamical system.
For example, for the Lorenz system, it needs 1,000,000 samples to achieve the modeling error of 0.197. Using our proposed critical sampling method, the number of samples can be reduced to  1,765, while achieving an even smaller modeling error 0.194. The number of samples has been reduced by 567 times. 
For the Viscous Burgers' PDE system, the number of samples is also reduced by 25 times. 

\renewcommand{\arraystretch}{1.1}
\begin{table*}[htbp]
% \vspace{-10pt}
 \caption{Training time and inference time of our method and baseline method in all the experiments.}
\vspace{-8pt}
\label{tab-time}
\begin{sc}
\begin{small}
\begin{center}
\resizebox{0.9\linewidth}{!}{
\begin{tabular}{llccc}
\toprule
System & Method & Samples & Training Time (s)& Inference Time (s)\\ \midrule
\multirow{2}{*}{Damped Pendulum} & Ours & 417 &    2554.8     &  0.314  \\
                                 & Baseline & 14400 &  749.7    &  0.296 \\ \midrule
\multirow{2}{*}{2D Nonlinear} & Ours & 925 &   1868.4   & 0.121  \\
                                 & Baseline & 14400 &  444.1   & 0.129 \\ \midrule  
\multirow{2}{*}{Lorenz System} & Ours & 1765 &   81473.0          &   0.403  \\
                                 & Baseline & 1000000 &  14936.1    &  0.384 \\ \midrule   
\multirow{2}{*}{Viscous Burgers' Eq.} & Ours & 19683 &   68263.6      &  0.483  \\
                                 & Baseline & 500000 &  12137.5  &  0.488 \\  
\bottomrule
\end{tabular}
}
\end{center}
\end{small}
\end{sc}
\vspace{-10pt}
\end{table*}

Figure \ref{fig:all-error-samples} shows the performance comparison results for the four dynamical systems listed in Table \ref{tab-eqs}. 
In each sub-figure, the horizontal dashed line shows the average network modeling error achieved by the baseline method for the number of samples shown in the legend. This number is empirically chosen since it is needed for the network to achieve a reasonably accurate and robust learning performance.  We can see that as more and more samples are selected by our critical sampling method, the network modeling error quickly drops below the average modeling error of the baseline method. 

Figure \ref{fig:ODE-error} shows that the network modeling errors ${\mathcal E}[\mathbf{u}(0)]$ of the Damped Pendulum system (top) and the 2D Nonlinear system (bottom), are being quickly reduced with more and more critical samples are collected. 
Figure \ref{fig:solution} shows the phase portraits of the solutions for two systems, the Damped Pendulum system (left) and the 2D Nonlinear system (right), obtained by our method with comparison against the ground-truth solutions. 
We can see that, using a few hundred samples, our method is able to accurately learn the system evolution patterns.
Figure \ref{fig:ode3solution} shows an example solution trajectory for the Lorenz system. The result shows that our method can perform robust prediction on chaotic systems with only thousands of samples.
Figure \ref{fig:pdesolution} shows an example solution for the Viscous Burgers' PDE system. The first one on the top left is the ground-truth solution. The figures in the second column are the solution and its difference from the ground-truth solution for the baseline method. The figures in the final column are the solution and difference obtained by our method. We can see that using fewer samples, our method is able to learn the system evolution and predict its future states at the same level of accuracy.

In Table \ref{tab-time}, we present the training time and inference time of our method and baseline method respectively in all 4 numerical experiments. 
Our method requires more training time compared to the baseline, as the training time directly depends on the number of iterations needed to reach the threshold for the network modeling error. However, our method does not add complexity during the inference stage.
Although the proposed method has higher training complexity, our method can dramatically reduce the number of needed samples, thereby reducing the sample collection cost. This is particularly crucial in real-world dynamical systems where collecting data can be expensive and time-consuming.

\subsection{More Experimental Results}
To systematically evaluate our proposed method, we provide an empirical analysis on the effects of the hyper-parameters in this
section. We also compare our method with the state-of-the-art Markov Neural Operator \cite{li2022learning} to further demonstrate the effectiveness of our proposed critical sampling method.

\textbf{(1) Experiments on different numbers of reciprocal steps $K$.} In our experiments on all four dynamical systems, we set the hyper-parameters $K$ to 5. To comprehensively
investigate the effects of different values of $K$, we also conduct a sensitivity experiment where we set $K$ to 1, 3, 5, 7, 9 and evaluate the performance
on the Damped Pendulum system. The ablation results are presented in Table \ref{tab:Kstep}. The prediction error varies from 0.02411 to 0.09562 when using 925 samples and $K = 5$ yields the optimal performance.

\begin{table}[ht]
\begin{center}
\begin{sc}
\caption{Sensitivity analysis of hyper-parameter $K$ on the Damped Pendulum system.}
\label{tab:Kstep}
\begin{tabular}{cccccc}
\toprule
Samples & $K=1$     & $K=3$     & $K=5$     & $K=7$     & $K=9$     \\ \midrule
225     & 0.29783 & 0.17622 & \textbf{0.12706} & 0.23542 & 0.42140 \\
297     & 0.18295 & 0.10267 & \textbf{0.07370} & 0.15671 & 0.23097 \\
333     & 0.11774 & 0.07934 & \textbf{0.04345} & 0.10925 & 0.15458 \\
417     & 0.06281 & 0.04166 & \textbf{0.02411} & 0.07320 & 0.09562 \\ \bottomrule
\end{tabular}
\vspace{-10pt}
\end{sc}
\end{center}
\end{table}

\textbf{(2) Experiments on different sampling frequencies.}
We conduct an additional experiment to demonstrate the performance of our method for different sampling frequencies and target errors. In this experiment, we select 3,600, 6,400, 10,000, and 14,400 uniformly distributed samples for the baseline methods to learn the evolution operators of the Damped Pendulum and the 2D Nonlinear systems. The performance comparisons are shown in Table \ref{tab:samples2}. As we can see from the table, our method is able to significantly reduce the number of needed training samples for the deep neural networks when modeling the unknown dynamical systems, no matter how many samples are used for the baseline method.

\renewcommand{\arraystretch}{1.1}
\begin{table}[ht]
\caption{Samples and errors for learning the system evolution using the baseline method and our method on Damped Pendulum and 2D Nonlinear ODE systems.}
\vspace{-10pt}
\label{tab:samples2}
\vspace{-0.1in}
\begin{center}
\begin{small}
\resizebox{\linewidth}{!}{
\begin{tabular}{l|cc|cc|c}
\hline
\multirow{2}{*}{Dynamical System} & \multicolumn{2}{c|}{Baseline} & \multicolumn{2}{c|}{Our Work} & \multirow{2}{*}{Ratio}  \\ \cline{2-5}
 & Samples & Error & Samples & Error & \\ \hline

\multirow{4}{*}{Damped Pendulum}  & {3600}          & {0.12803}     & \textbf{225}      & \textbf{0.12706} & {16.00} \\
                                  & {6400}          & {0.07459}     & \textbf{297}      & \textbf{0.07370} & {21.55} \\
                                  & {10000}         & {0.04370}     & \textbf{333}      & \textbf{0.04345} & {30.03} \\
                                  & {14400}         & {0.02630}     & \textbf{417}      & \textbf{0.02411} & {34.53} \\ \hline
\multirow{4}{*}{2D Nonlinear}     & {3600}          & {0.00695}     & \textbf{496}      & \textbf{0.00673} & {7.26}  \\
                                  & {6400}          & {0.00254}     & \textbf{625}      & \textbf{0.00250} & {10.24} \\
                                  & {10000}         & {0.00057}     & \textbf{825}      & \textbf{0.00045} & {12.12} \\
                                  & {14400}         & {0.00037}     & \textbf{925}      & \textbf{0.00035} & {15.57} \\\hline

\end{tabular}
}
\vspace{-0.1in}
\end{small}
\end{center}
\end{table}

\textbf{(3) Comparison with Markov Neural Operator.} 
To demonstrate the effectiveness of the critical sampling approach, we conduct a comparison with Markov Neural Operator (MNO) \cite{li2022learning}, a state-of-the-art method for learning chaotic systems, including the three-dimensional Lorenz system. To ensure a fair comparison, we follow the experimental settings of MNO \cite{li2022learning}, training a 6-layer feedforward neural network with 200,000 data pairs from a single trajectory with a time step of $\Delta=0.05$ seconds. We evaluate the learned evolution operators based on the relative (normalized) $l^2$-error over 1 second and report both the per-step error and per-second error. The per-step error is the error on the timescale $\Delta=0.05$ seconds used in training, while the per-second error is the error of the model composed with itself 20 times. We then use our method to adaptively discover critical samples to reach the same error. The experimental results are summarized in Table \ref{tab:mno}. Here, the results of MNO are directly taken from the original paper \cite{li2022learning} of MNO. In this table, MNO (w/ diss.) indicates that dissipativity is enforced during the training process of MNO. See \cite{li2022learning} for more details.
With 45 times fewer training samples than MNO, our method can still achieve an even smaller per-step and per-second prediction errors.

\begin{table}[ht]
\begin{center}
\begin{sc}
\caption{Performance comparison of our method and Markov Neural Operator on the Lorenz system.}
\label{tab:mno}
\resizebox{\linewidth}{!}{
\begin{tabular}{lccc}
\toprule
Methods & Samples     & Per-step     & Per-second \\ \midrule
MNO (w/o diss.)    & 200000 &0.000570 & 0.0300 \\
MNO (w/ diss.)    & 200000 & 0.000564 & 0.0264 \\
Ours    & \textbf{4452} & \textbf{0.000559} & \textbf{0.0261} \\ \bottomrule
\end{tabular}
}
\vspace{-10pt}
\end{sc}
\end{center}
\end{table}

\section{Conclusion and Further Discussion}
\label{sec-conclusion}
In this work, we have studied the critical sampling for the adaptive evolution operator learning problem.
We have made an interesting finding that the network modeling error is correlated with the multi-step reciprocal prediction error.  
With this, we are able to perform a dynamic selection of critical samples from regions with high network modeling errors and develop an adaptive sampling-learning method for dynamical systems based on the spatial-temporal evolution network.
Extensive experimental results demonstrate that our method is able to dramatically reduce the number of samples needed for effective learning and accurate prediction of the evolution behaviors.

In the future, we hope to apply our approach to large-scale dynamical systems, by 
combining some reduced-order modeling or lifting techniques (cf.~\cite{qian2020lift}) or incorporating certain sparsity (cf.~\cite{schaeffer2018extracting}). 
Another important question that has not been fully addressed in this paper is how to control the system state towards those samples selected by our critical sampling method. During simulations, this system state control is often available. However, for some complex systems, the exact change of the system state may not be trivial. In this case, we shall investigate how the system state control impacts the critical sampling and system modeling performance.

\bibliographystyle{IEEEtran}
\bibliography{reference_abbr}
\begin{IEEEbiography}[{\includegraphics[width=1in,height=1.25in,clip,keepaspectratio]{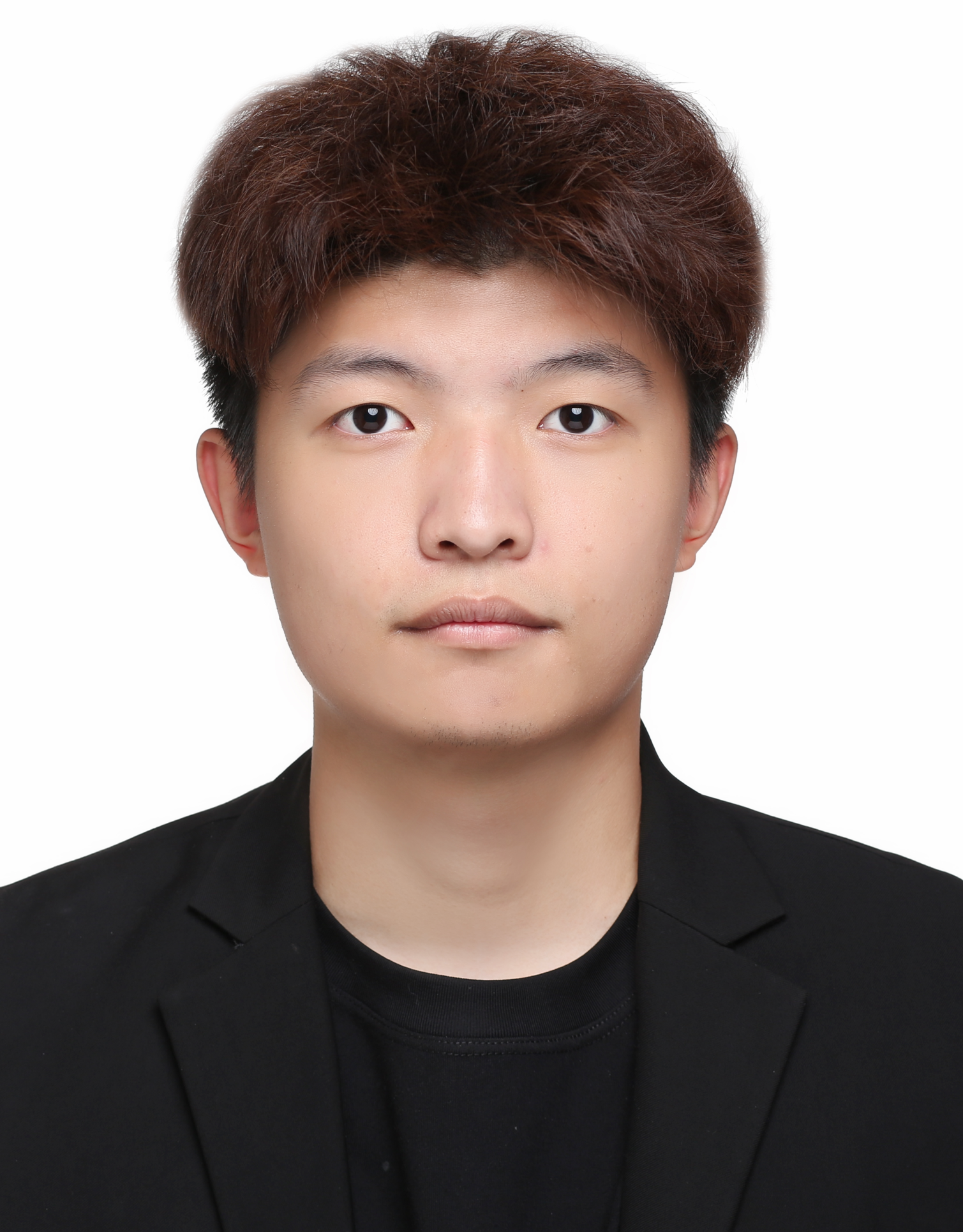}}]{Ce Zhang}{\space}received the B.Eng. degree in communication engineering from the Department of Electronic and Electrical Engineering, Southern University of Science and Technology, Shenzhen, Guangdong, China, in 2023. He is currently pursuing the M.S. degree in machine learning at the Machine Learning Department, Carnegie Mellon University, Pittsburgh, PA, USA. His research interests include machine learning and computer vision.
\end{IEEEbiography}

\enlargethispage{-15\baselineskip}

\begin{IEEEbiography}[{\includegraphics[width=1in,height=1.25in,clip,keepaspectratio]{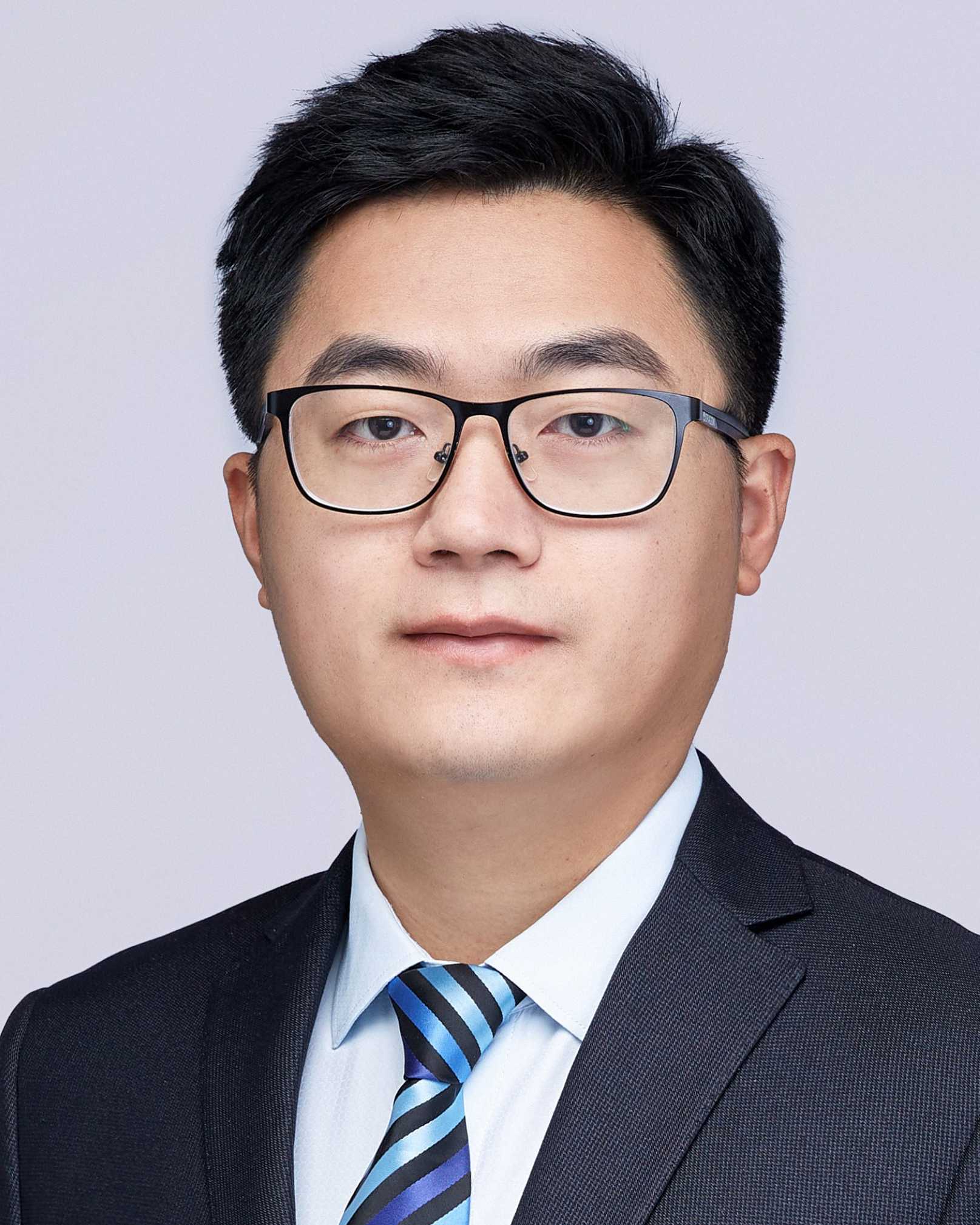}}]{Kailiang Wu}{\space}received his B.S. from Huazhong University of Science and Technology in 2011 and Ph.D. from Peking University in 2016. From 2016 to 2020, he was a postdoctoral scholar at the University of Utah and the Ohio State University. He currently serves as an associate professor at the Department of Mathematics and International Center for Mathematics at the Southern University of Science and Technology and is affiliated with the National Center for Applied Mathematics Shenzhen. His research focuses on machine learning, data-driven modeling, numerical methods for partial differential equations, and computational fluid dynamics, etc.
\end{IEEEbiography}

\begin{IEEEbiography}[{\includegraphics[width=1in,height=1.25in,clip,keepaspectratio]{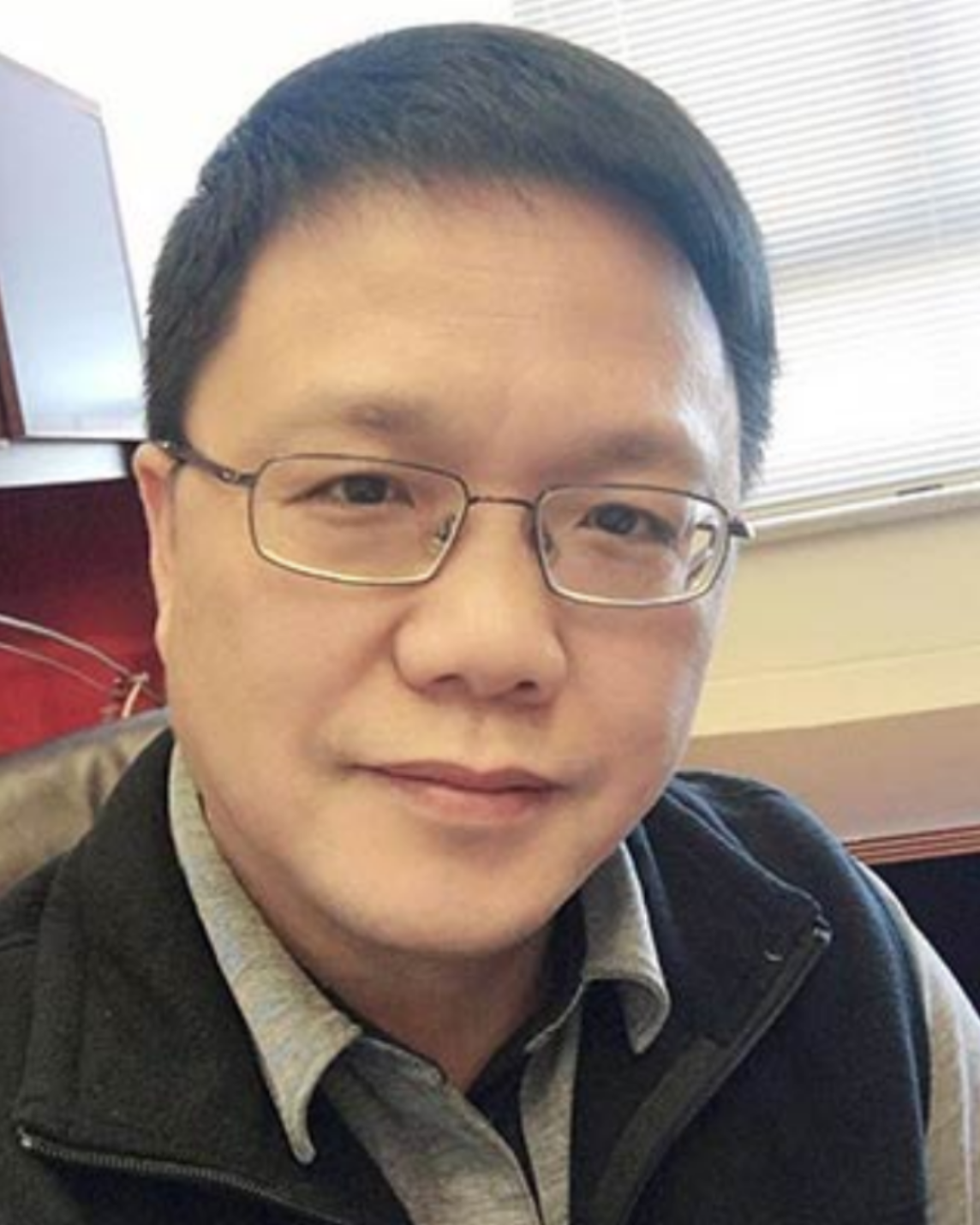}}]{Zhihai He} (IEEE Fellow 2015) received the B.S. degree in mathematics from Beijing Normal University, Beijing, China, in 1994, the M.S. degree in mathematics from the Institute of Computational Mathematics, Chinese Academy of Sciences, Beijing, China, in 1997, and the Ph.D. degree in electrical engineering from the University of California, at Santa Barbara, CA, USA, in 2001. In 2001, he joined Sarnoff Corporation, Princeton, NJ, USA, as a member of technical staff. In 2003, he joined the Department of Electrical and Computer Engineering, University of Missouri, Columbia, MO, USA, where he was a tenured full professor. He is currently a chair professor with the Department of Electrical and Electronic Engineering, Southern University of Science and Technology, Shenzhen, P. R. China. His current research interests include image/video processing and compression, wireless sensor network, computer vision, and cyber-physical systems.

He is a member of the Visual Signal Processing and Communication Technical Committee of the IEEE Circuits and Systems Society. He serves as a technical program committee member or a session chair of a number of international conferences. He was a recipient of the 2002 {\sc IEEE Transactions on Circuits and Systems for Video Technology} Best Paper Award and the SPIE VCIP Young Investigator Award in 2004. He was the co-chair of the 2007 International Symposium on Multimedia Over Wireless in Hawaii. He has served as an Associate Editor for the {\sc IEEE Transactions on Circuits and Systems for Video Technology} (TCSVT), the {\sc IEEE Transactions on Multimedia} (TMM), and the Journal of Visual Communication and Image Representation. He was also the Guest Editor for the IEEE TCSVT Special Issue on Video Surveillance.
\end{IEEEbiography}

\end{document}